\documentclass{article}

\usepackage[final]{nips_2016}
\setcitestyle{numbers,square,comma}
\usepackage[utf8]{inputenc} % allow utf-8 input
\usepackage[T1]{fontenc}    % use 8-bit T1 fonts
\usepackage{hyperref}       % hyperlinks
\usepackage{url}            % simple URL typesetting
\usepackage{booktabs}       % professional-quality tables
\usepackage{amsfonts,amsmath,amsthm,amssymb}       % blackboard math symbols
\usepackage{nicefrac}       % compact symbols for 1/2, etc.
\usepackage{microtype}      % microtypography
\usepackage{graphicx,color,enumitem}
\title{Robustness of classifiers: from adversarial to random noise}

\author{
Alhussein Fawzi\thanks{The first two authors contributed equally to this work.}\\
  \'Ecole Polytechnique F\'ed\'erale de Lausanne \\
  Lausanne, Switzerland \\
  \texttt{alhussein.fawzi at epfl.ch} \\  
  \And
  Seyed-Mohsen Moosavi-Dezfooli\footnotemark[1]
    \\
  \'Ecole Polytechnique F\'ed\'erale de Lausanne\\
  Lausanne, Switzerland\\
  \texttt{seyed.moosavi at epfl.ch} \\
  \AND
  Pascal Frossard \\
   \'Ecole Polytechnique F\'ed\'erale de Lausanne\\
 Lausanne, Switzerland\\
  \texttt{pascal.frossard at epfl.ch} \\
}

\usepackage{graphicx}
\usepackage{subfigure}
\usepackage{amsfonts,amssymb,amsmath,amsthm}
\usepackage{color}
\usepackage{url}
\usepackage{tablefootnote}
\usepackage{mathrsfs}
\usepackage{enumitem}
\usepackage{array}
\usepackage{algorithm,algpseudocode}
\usepackage{wrapfig}
\usepackage{xcolor}
\usepackage{mathtools}

\usepackage{nicefrac}
\usepackage{eufrak}
\usepackage{yfonts}

\newlength\myindent
\graphicspath{{/Users/hussein/Dropbox/thesis/thesis/images/chapterLAST/}{/Users/hussein/Dropbox/thesis/thesis/images/chapterTheory/}{/Users/hussein/Dropbox/thesis/thesis/images/chapterDF/}{/Users/hussein/Dropbox/thesis/thesis/images/chapterManitest/}{/Users/hussein/Dropbox/thesis/thesis/images/chapterNuisance/}{/Users/hussein/Dropbox/thesis/thesis/images/chapterPriorArt/}{/Users/hussein/Dropbox/thesis/thesis/images/chapterSemiRandom/}}

\newcommand{\argmax}{\operatornamewithlimits{argmax}}
\newcommand{\argmin}{\operatornamewithlimits{argmin}}

\newtheorem{lemma}{Lemma}
\newtheorem{theorem}{Theorem}

\newtheorem{corollary}{Corollary}

\newcommand{\bb}[1]{\mathbb{#1}}

\newcommand{\bo}{\mathscr{B}}
\newcommand{\gaone}{\zeta_1(m, \delta)}
\newcommand{\gatwo}{\zeta_2(m, \delta)}
\def\x{\boldsymbol{x}}
\def\y{\boldsymbol{y}}
\def\z{\boldsymbol{z}}

\def\r{\boldsymbol{r}}
\def\w{\boldsymbol{w}}
\def\v{\boldsymbol{v}}
\def\p{\boldsymbol{p}}

\newcommand{\lab}{\hat{k}}
\newcommand{\numClass}{L}

\def\u{\boldsymbol{u}}
\def\S{\mathcal{S}}
\def\T{\mathcal{T}}

\def\R{\mathbb{R}}
\def\P{\mathbf{P}}
\def\Pbb{\mathbb{P}}

\def\B{\mathscr{B}}

\makeatletter
\newtheorem*{rep@theorem}{\rep@title}
\newcommand{\newreptheorem}[2]{%
\newenvironment{rep#1}[1]{%
 \def\rep@title{#2 \ref{##1}}%
 \begin{rep@theorem}}%
 {\end{rep@theorem}}}
\makeatother

\newreptheorem{theorem}{Theorem}
\newreptheorem{lemma}{Lemma}
\newreptheorem{corollary}{Corollary}

\DeclarePairedDelimiterX{\infdivx}[2]{}{}{%
  #1\;\delimsize\|\;#2%
}

\begin{document}

\maketitle

\begin{abstract}

Several recent works have shown that state-of-the-art classifiers are vulnerable to worst-case (i.e., adversarial) perturbations of the datapoints. On the other hand, it has been empirically observed that these same classifiers are relatively robust to random noise. In this paper, we propose to study a \textit{semi-random} noise regime that generalizes both the random and worst-case noise regimes. We propose the first quantitative analysis of the robustness of nonlinear classifiers in this general noise regime. We establish precise theoretical bounds on the robustness of classifiers in this general regime, which depend on the curvature of the classifier's decision boundary. Our bounds confirm and quantify the empirical observations that classifiers satisfying curvature constraints are robust to random noise. Moreover, we quantify the robustness of classifiers in terms of the subspace dimension in the semi-random noise regime, and show that our bounds remarkably interpolate between the worst-case and random noise regimes. We perform experiments and show that the derived bounds provide very accurate estimates when applied to various state-of-the-art deep neural networks and datasets. This result suggests bounds on the curvature of the classifiers' decision boundaries that we support experimentally, and more generally offers important insights onto the geometry of high dimensional classification problems.
\end{abstract}

% curvature of decision boundaries are sufficiently small

%!TEX root=nips_2016.tex
\section{Introduction}

State-of-the-art classifiers, especially deep networks, have shown impressive classification performance on many challenging benchmarks in visual tasks \cite{cv2} and speech processing \cite{sp2}. An equally important property of a classifier that is often overlooked is its \textit{robustness} in noisy regimes, when data samples are perturbed by noise. The robustness of a classifier is especially fundamental when it is deployed in real-world, uncontrolled, and possibly hostile environments. In these cases, it is crucial that classifiers exhibit good robustness properties. In other words, a sufficiently small perturbation of a datapoint should ideally not result in altering the estimated label of a classifier. State-of-the-art deep neural networks have recently been shown to be very unstable to worst-case perturbations of the data (or equivalently, \textit{adversarial} perturbations)~\citep{szegedy2013intriguing}. In particular, despite the excellent classification performances of these classifiers, well-sought perturbations of the data can easily cause misclassification, since data points often lie very close to the decision boundary of the classifier. Despite the importance of this result, the \textit{worst-case} noise regime that is studied in~\citep{szegedy2013intriguing} only represents a very specific type of noise. It furthermore requires the full knowledge of the classification model, which may be a hard assumption in practice. 
% Moreover, it is not clear whether the instability of these classifiers to worst-case perturbations implies instability in other noise regimes. % whether these classifiers are also unstable in other more realistic noise regimes.
% In particular, it does not tell us anything on the robustness to random or semi-random noise of the classifier. We fill this gap in this paper and study....
% The goal of this paper is to quantify

% In this paper, we perform a theoretical study of the robustness of general nonlinear classifiers in two key practical noise regimes. I

In this paper, we precisely quantify the robustness of nonlinear classifiers in two practical noise regimes, namely random and semi-random noise regimes. % perform a theoretical study of the robustness of general nonlinear classifiers in two key practical noise regimes.
In the \textit{random noise} regime, datapoints are perturbed by noise with random direction in the input space. The \textit{semi-random} regime generalizes this model to random \textit{subspaces} of arbitrary dimension, where a worst-case perturbation is sought within the subspace. In both cases, we derive bounds that precisely describe the robustness of classifiers in function of the \textit{curvature} of the decision boundary. % quantify the relations between the different noise regimes
% the robustness of the classifier that precisely describe the behaviour of a classifier to noisy samples in function of the \textit{curvature} of the decision boundary, and quantify the relations between the different noise regimes. 
We summarize our contributions as follows:
\begin{itemize}
\item In the random regime, we show that the robustness of classifiers behaves as $\sqrt{d}$ times the distance from the datapoint to the classification boundary (where $d$ denotes the dimension of the data) provided the curvature of the decision boundary is sufficiently small. This result highlights the blessing of dimensionality for classification tasks, as it implies that robustness to random noise in high dimensional classification problems can be achieved, even at datapoints that are very close to the decision boundary.
\item This quantification notably extends to the general semi-random regime, where we show that the robustness precisely behaves as $\sqrt{\nicefrac{d}{m}}$ times the distance to boundary, with $m$ the dimension of the subspace. This result shows in particular that, even when $m$ is chosen as a small fraction of the dimension $d$, it is still possible to find \textit{small} perturbations that cause data misclassification. 
\item We empirically show that our theoretical estimates are very accurately satisfied by state-of-the-art deep neural networks on various sets of data. This in turn suggests quantitative insights on the curvature of the decision boundary that we support experimentally through the visualization and estimation on two-dimensional sections of the boundary. % This offers quantitative insights onto the % of the curvature on two-dimensional sections of the decision boundary. In particular, our experiments seem to suggest that  % This suggests properties on the curvature of the decision boundary, that we support experimentally.
% Such experiments properties on the curvature of the decision boundary, that we verify experimentally.
% We perform experiments on state-of-the art classifiers and different datasets, and show that our theoretical estimates are very accurately OUR ESTIMATES ARE VERY accurately satisfies by on various sets of data
% We show that the derived bounds we established for sufficiently small curvature are precisely met by state-of-the-art classifiers, which suggests properties on the curvature of the decision boundary, that we verify experimentally.
\end{itemize}

The robustness of classifiers to noise has been the subject of intense research. The robustness properties of SVM classifiers have been studied in \cite{xu2009robustness} for example, and robust optimization approaches for constructing robust classifiers have been proposed to minimize the worst possible empirical error under noise disturbance \cite{sra2012optimization, lanckriet2003robust}. More recently, following the recent results on the instability of deep neural networks to worst-case perturbations \cite{szegedy2013intriguing}, several works have provided explanations of the phenomenon \cite{fawzi2015a,goodfellow2014, sabour2016adversarial,tabacof2015exploring}, and designed more robust networks \citep{gu2014towards, huang2015learning, zhao2016suppressing, moosavi2015deepfool, shaham2015understanding,luo2015foveation}. In \citep{tabacof2015exploring}, the authors provide an interesting empirical analysis of the adversarial instability, and show that adversarial examples are not isolated points, but rather occupy dense regions of the pixel space. In \cite{bmvc2015_106}, state-of-the-art classifiers are shown to be vulnerable to geometrically constrained adversarial examples. Our work differs from these works, as we provide a theoretical study of the robustness of classifiers to random and semi-random noise in terms of the robustness to adversarial noise. In \cite{fawzi2015a}, a formal relation between the robustness to random noise, and the worst-case robustness is established in the case of linear classifiers. Our result therefore generalizes \cite{fawzi2015a} in many aspects, as we study general nonlinear classifiers, and robustness to semi-random noise. Finally, it should be noted that the authors in \cite{goodfellow2014} conjecture that the ``high linearity'' of classification models explains their instability to adversarial perturbations. The objective and approach we follow here is however different, as we study theoretical relations between the robustness to random, semi-random and adversarial noise. 

\section{Definitions and notations}

\label{sec:definition_semirandom}
Let $f:\mathbb{R}^d\rightarrow\mathbb{R}^L$ be an $\numClass$-class classifier. Given a datapoint $\x_0 \in \mathbb{R}^d$, the estimated label is obtained by $\lab(\x_0)=\argmax_k f_k(\x_0)$, where $f_k(\x)$ is the $k^{\text{th}}$ component of $f(\x)$ that corresponds to the $k^{\text{th}}$ class.
Let $\S$ be an arbitrary subspace of $\mathbb{R}^d$ of dimension $m$. Here, we are interested in quantifying the robustness of $f$ with respect to different noise regimes. To do so, we define $\r_{\S}^*$ to be the perturbation in $\S$ of minimal norm that is required to change the estimated label of $f$ at $\x_0$.\footnote{Perturbation vectors sending a datapoint exactly to the boundary are assumed to change the estimated label of the classifier.}
% \footnote{For simplicity, we will assume that the ground truth label of $\x_0$ is equal to the estimated label, i.e., $\hat{k} (\x_0)$.} d
\begin{equation}
\begin{split}
\r_{\S}^*(\x_0)=&\argmin_{\r \in \S} \|\r\|_2 \text{ s.t. } \hat{k} (\x_0+\r) \neq \hat{k} (\x_0).
\end{split}\label{eq:def_subspace_optimal}
\end{equation}
%\begin{equation}
%\begin{split}
%\r_{\S}^*=&\argmin_{\r \in \S} \|\r\|_2\\
%&\text{s.t. } \exists k\neq\hat{k}(\x_0): f_k(\x_0+\r)\geq f_{\lab}(\x_0+\r).
%\end{split}\label{eq:def_subspace_optimal}
%\end{equation}
Note that $\r_{\S}^*(\x_0)$ can be equivalently written 
\begin{align}
\label{eq:second_form_robustness}
\r_{\S}^*(\x_0) = \argmin_{\r \in \S} \|\r\|_2 \text{ s.t. } \exists k\neq\hat{k}(\x_0): f_k(\x_0+\r)\geq f_{\lab(\x_0)}(\x_0+\r).
\end{align}
When $\S = \mathbb{R}^d, \r^*(\x_0) := \r_{\mathbb{R}^d}^*(\x_0)$ is the \textit{adversarial (or worst-case) perturbation}  defined in \cite{szegedy2013intriguing}, which corresponds to the (unconstrained) perturbation of minimal norm that changes the label of the datapoint $\x_0$. In other words, $\| \r^*(\x_0) \|_2$ corresponds to the minimal distance from $\x_0$ to the classifier boundary. % $\r^*$ is also called an \textit{adversarial} perturbation~\cite{szegedy2013}, as it corresponds to the minimal noise that an adversary having knowledge of the classifier model needs to apply in order to misclassify the datapoint. 
In the case where $\S \subset \mathbb{R}^d$, only perturbations along $\S$ are allowed. The robustness of $f$ at $\x_0$ along $\S$ is naturally measured by the norm $\| \r_{\S}^*(\x_0) \|_2$. Different choices for $\S$ permit to study the robustness of $f$ in two different regimes:

\begin{itemize}% [noitemsep, nolistsep]
\item \textbf{Random noise regime}: This corresponds to the case where $\S$ is a \textit{one-dimensional subspace} ($m=1$) with direction $\boldsymbol{v}$, where $\boldsymbol{v}$ is a \textit{random vector} sampled uniformly from the unit sphere $\mathbb{S}^{d-1}$. Writing it explicitly, we study in this regime the robustness quantity defined by $\min_{t} |t| \text{ s.t. } \exists k \neq \lab(\x_0), f_k(\x_0+t \boldsymbol{v}) \geq f_{\lab(\x_0)} (\x_0 + t \boldsymbol{v})$, where $\boldsymbol{v}$ is a vector sampled uniformly at random from the unit sphere $\mathbb{S}^{d-1}$.
\item \textbf{Semi-random noise regime}: In this case, the subspace $\S$ is chosen \textit{randomly}, but can be of arbitrary dimension $m$.\footnote{A random subspace is defined as the span of $m$ independent vectors drawn uniformly at random from $\mathbb{S}^{d-1}$.} We use the \textit{semi}-random terminology as the subspace is chosen randomly, and the smallest vector that causes misclassification is then sought in the subspace.
% and seek a vector in the subspace that causes misclassification.
% but a worst-case direction in that subspace is chosen.
It should be noted that the random noise regime is a special case of the semi-random regime with a subspace of dimension $m = 1$. We differentiate nevertheless between these two regimes for clarity. 
%as the perturbation direction is chosen in a worst-case fashion for subspaces of dimension larger than $1$, whereas the random noise regime does not involve any choice in the direction.}
% In other words, we study here the magnitude of random noise one can apply to a datapoint without causing data misclassification.
\end{itemize}

In the remainder of the paper, the goal is to establish relations between the robustness in the random and semi-random regimes on the one hand, and the robustness to adversarial perturbations $\|\r^*(\x_0)\|_2$ on the other hand. We recall that the latter quantity captures the distance from $\x_0$ to the classifier boundary, and is therefore a key quantity in the analysis of robustness.

In the following analysis, we fix $\x_0$ to be a datapoint classified as $\lab(\x_0)$. To simplify the notation, we remove the explicit dependence on $\x_0$ in our notations (e.g., we use $\r^*_\S$ instead of $\r^*_\S (\x_0)$ and $\lab$ instead of $\lab(\x_0)$), and it should be implicitly understood that all our quantities pertain to the fixed datapoint $\x_0$.

\section{Robustness of affine classifiers}
\label{sec:robustness_affine_semirandom}

We first assume that $f$ is an affine classifier, i.e., $f(\x)=\mathbf{W}^\top\x+\boldsymbol{b}$ for a given $\mathbf{W} = [\w_1 \dots \w_\numClass]$ and $\boldsymbol{b} \in \mathbb{R}^L$.

% \red{maybe say somewhere that you remove the dependence on datapoint $\x_0$ and classification function, to make it more readable}

% \subsection{Notations and definitions}
% We first introduce several notations used in our analysis. Let  We define

% \subsection{Affine classifier}

% \subsection{Main result}
The following result shows a precise relation between the robustness to semi-random noise, $\|\r^*_\S\|_2 $ and the robustness to adversarial perturbations, $\|\r^*\|_2 $.% in the random and semi-random noise regimes.

\begin{theorem}
\label{th:main_result_linear}
Let $\delta>0$ and $\S$ be a random $m$-dimensional subspace of $\mathbb{R}^d$, and $f$ be a $L$-class affine classifier. Let
\begin{align}
\label{eq:chap_semirandom_gaone}
\gaone & = \left(1 + 2 \sqrt{\frac{ \ln(1/\delta)}{m}} + \frac{2 \ln(1/\delta)}{m} \right)^{-1}, \\
\label{eq:chap_semirandom_gatwo}
\gatwo & = \left( \max \left( (1/e) \delta^{2/m}, 1-\sqrt{2(1-\delta^{2/m})} \right) \right)^{-1}.
\end{align}
The following inequalities hold between the robustness to semi-random noise $\|\r^*_\S\|_2 $, and the robustness to adversarial perturbations $\|\r^*\|_2 $:
% \|\r_\S^*(\x_0) \|_2^2
\begin{equation}
\sqrt{\gaone}  \sqrt{\frac{d}{m}} \|\r^*\|_2  \leq
\|\r^*_\S\|_2  \leq \sqrt{\gatwo} \sqrt{\frac{d}{m}} \|\r^*\|_2 ,
\end{equation}
with probability exceeding $1 - 2 (L+1) \delta$.
\label{th:linear}
\end{theorem}

\begin{figure}[ht]
\center
\includegraphics[scale=0.4]{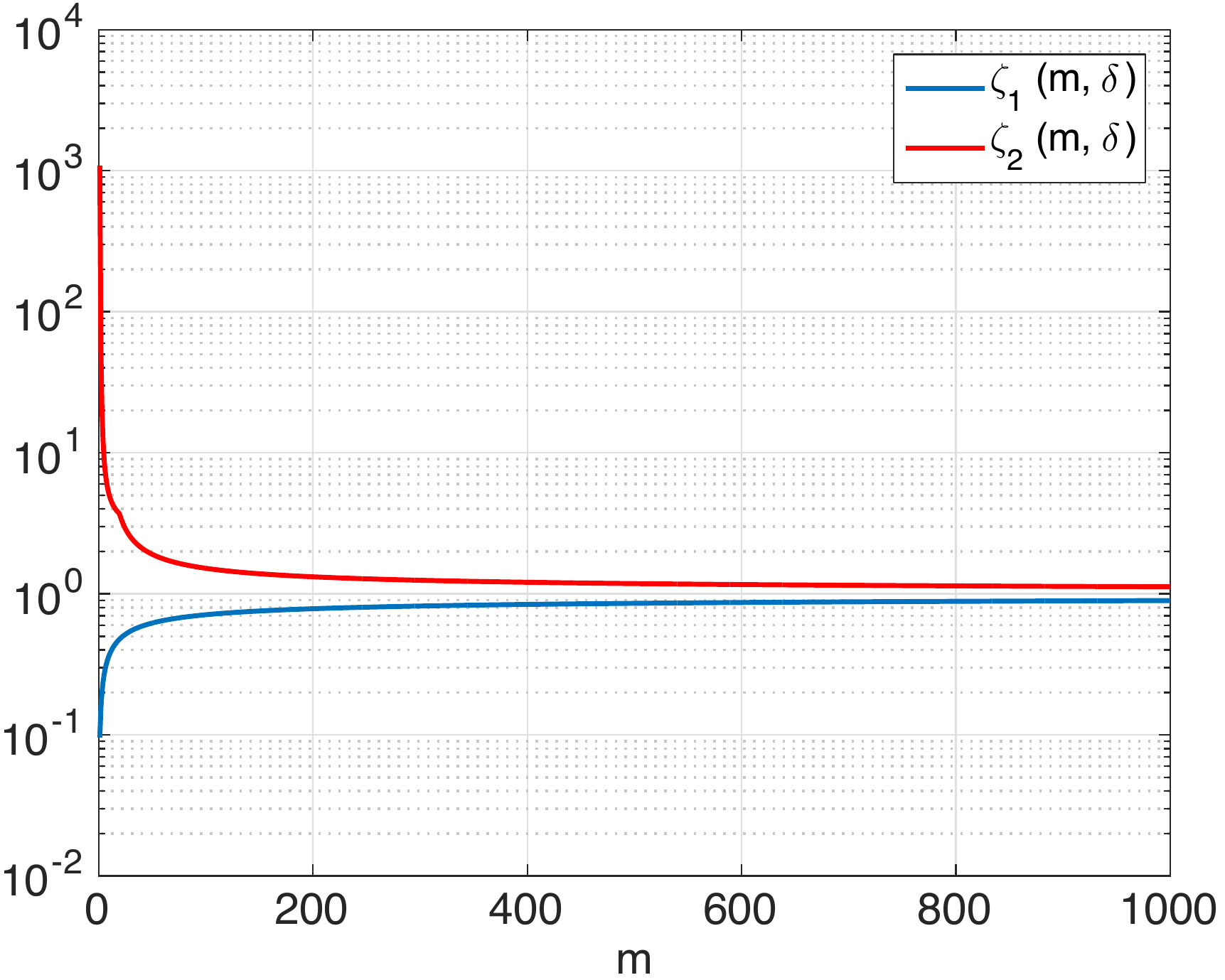}
\caption{$\gaone$ and $\gatwo$ in function of $m$ [$\delta = 0.05$] .}
\label{fig:zetas}
\end{figure}

The proof can be found in the appendix.
Our upper and lower bounds depend on the functions $\gaone$ and $\gatwo$ that control the inequality constants (for $m$, $\delta$ fixed). It should be noted that $\gaone$ and $\gatwo$ are independent of the data dimension $d$. Fig.~\ref{fig:zetas} shows the plots of $\gaone$ and $\gatwo$ as functions of $m$, for a fixed $\delta$. It should be noted that for sufficiently large $m$, $\gaone$ and $\gatwo$ are very close to $1$ (e.g., $\gaone$ and $\gatwo$ belong to the interval $[0.8, 1.3]$ for $m \geq 250$ in the settings of Fig.~\ref{fig:zetas}). The interval $[\gaone, \gatwo]$ is however (unavoidably) larger when $m=1$.

The result in Theorem \ref{th:linear} shows that in the random and semi-random noise regimes, the robustness to noise is precisely related to $\|\r^*\|_2$ by a factor of $\sqrt{\nicefrac{d}{m}}$. Specifically, in the random noise regime ($m = 1$), the magnitude of the noise required to misclassify the datapoint behaves as $\Theta( \sqrt{d} \|\r^*\|_2 )$ with high probability, with constants in the interval $[\zeta_1(1, \delta), \zeta_2(1, \delta)]$. Our results therefore show that, in high dimensional classification settings, affine classifiers can be robust to random noise, even if the datapoint lies very closely to the decision boundary (i.e., $\|\r^*\|_2$ is small). % \red{Refer to Szegedy later.}% This theoretically explains the results reported previously in \citep{szegedy2013intriguing}, on the large robustness of classifiers to random noise, despite having a very small robustness to adversarial perturbation . 
In the semi-random noise regime with $m$ sufficiently large (e.g., $m \geq 250$), we have $\|\r^*_\S\|_2 \approx \sqrt{\nicefrac{d}{m}} \|\r^*\|_2$ with high probability, as the constants $\zeta_1(m, \delta) \approx \zeta_2(m, \delta) \approx 1$ for sufficiently large $m$. % \red{figure needed}
Our bounds therefore ``interpolate'' between the random noise regime, which behaves as $\sqrt{d} \|\r^*\|_2$, and the worst-case noise $\|\r^*\|_2$. More importantly, the square root dependence is also notable here, as it shows that the semi-random robustness can remain small even in regimes where $m$ is chosen to be a very small fraction of $d$. For example, choosing a small subspace of dimension $m = 0.01 d$ results in semi-random robustness of $10 \|\r^*\|_2$ with high probability, which might still not be perceptible in complex visual tasks. Hence, for semi-random noise that is mostly random and only mildly adversarial (i.e., the subspace dimension is small), affine classifiers remain vulnerable to such noise.  %  (see experimental section for visual results).

%%%%%%%%%%%%%%%%%%%%%% BEGIN PROOF %%%%%%%%%%%%%%%%%%%%%%%

%%%%%%%%%%%%%%%%%%%%%%%%%%%% END PROOF %%%%%%%%%%%%%%%%%

\section{Robustness of general classifiers}
\label{sec:robustness_general_classifiers_semirandom}

% \red{Make sure in the main text to say that we remove the dependence to $\x_0$ for simplicity of the notation, and say that we work for a fixed $\x_0$, hence the class is also determined by the class of the datapoint, and we write $\bo_k$ to denote actually $\bo_{lab, k}$.}

\subsection{Decision boundary curvature}

We now consider the general case where $f$ is a nonlinear classifier. We derive relations between the random and semi-random robustness $\|\r^*_\S\|_2 $ and worst-case robustness $\|\r^*\|_2 $ using properties of the classifier's \textit{boundary}. Let $i$ and $j$ be two arbitrary classes; we define the pairwise boundary $\B_{i,j}$ as the boundary of the \textit{binary} classifier where only classes $i$ and $j$ are considered. Formally, the decision boundary $\B_{i,j}$ reads as follows:
\begin{align*}
\B_{i,j} = \{ \x \in \R^d: f_{i} (\x) - f_{j} (\x) = 0  \}.
\end{align*}
The boundary $\B_{i,j}$ separates between two regions of $\R^d$, namely $\mathcal{R}_i$ and $\mathcal{R}_j$, where the estimated label of the binary classifier is respectively $i$ and $j$. Specifically, we have % \red{should be strict here actually because there might be a point with $f_i - f_j = 0$ inside the ball, but then the next definition cannot hold as $\p \neq \mathcal{R}_i$.}
\begin{align*}
\mathcal{R}_i & = \{ \x \in \R^d: f_{i} (\x) > f_{j} (\x) \}, \\
\mathcal{R}_j & = \{ \x \in \R^d: f_{j} (\x) > f_{i} (\x) \}.
\end{align*}

\begin{figure}[ht]
\centering
\includegraphics[scale=0.4]{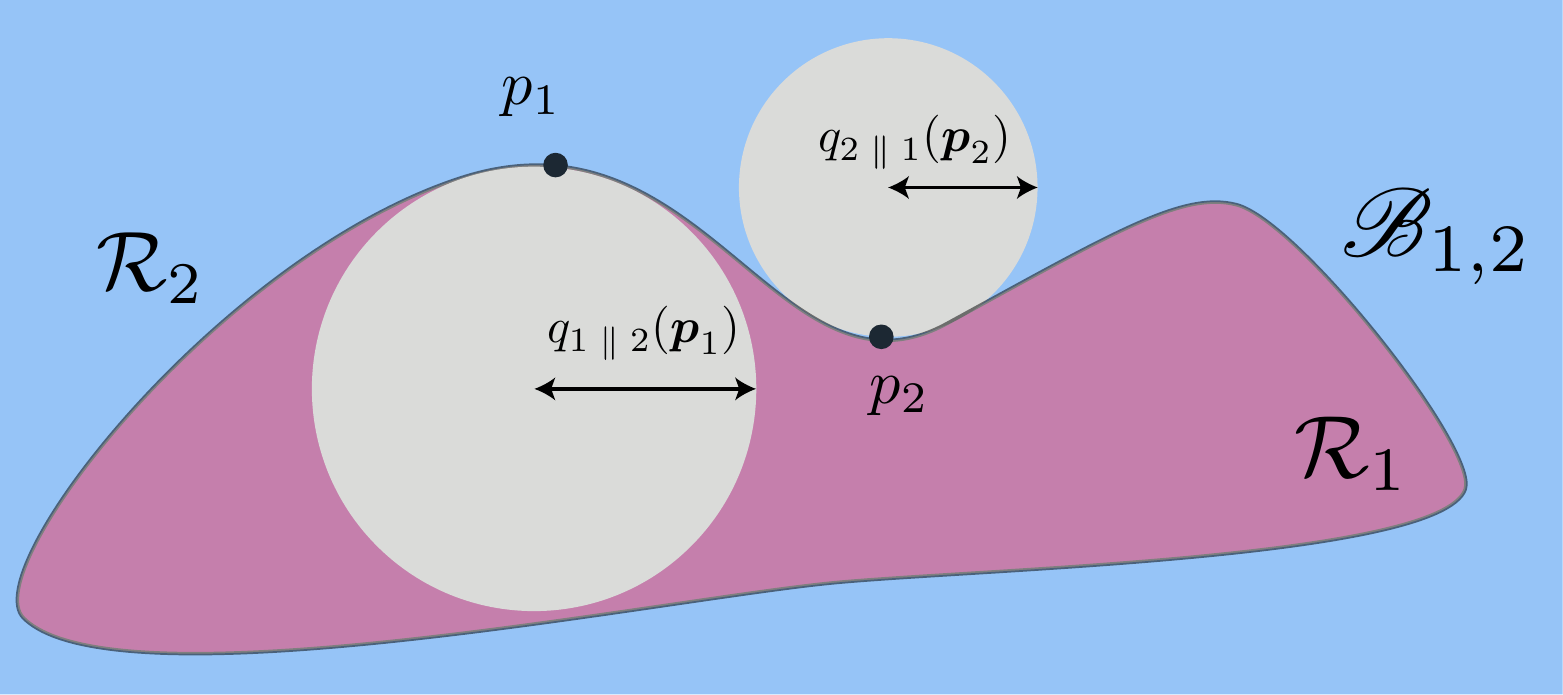}
\caption{\label{fig:radius_schema} Illustration of the quantities introduced for the definition of the curvature of the decision boundary.}
\end{figure}

We assume for the purpose of this analysis that the boundary $\bo_{i,j}$ is smooth. We are now interested in the geometric properties of the boundary, namely its curvature. There are many notions of curvature that one can define on hypersurfaces \cite{lee2009manifolds}. In the simple case of a curve in a two-dimensional space, the curvature is defined as the inverse of the radius of the so-called oscullating circle. One way to define curvature for high-dimensional hypersurfaces is by taking \textit{normal} sections of the hypersurface, and looking at the curvature of the resulting planar curve (see Fig. \ref{fig:subadv_nonlin}). We however introduce a notion of curvature that is specifically suited to the analysis of the decision boundary of a classifier. Informally, our curvature captures the \textit{global} bending of the decision boundary by inscribing balls in the regions separated by the decision boundary. % Thus, we define the \textit{curvature} of the decision boundary, denoted $\kappa\left(\B_{i,j}\right)$, which measures how $\B_{i,j}$ bends in different directions in the space $\R^d$. 

We now formally define this notion of curvature. For a given $\p \in \B_{i,j}$, we define $q_{\infdivx{i}{j}} (\p)$ to be the radius of the largest open ball included in the region $\mathcal{R}_i$ that intersects with $\B_{i,j}$ at $\p$; i.e., % \red{we should probably say open ball here...}
\begin{align}
\label{eq:radius_ball}
q_{\infdivx{i}{j}} (\p) = \sup_{\z \in \R^d} \left\{ \| \z - \p \|_2: \mathbb{B} (\z, \| \z - \p \|_2) \subseteq \mathcal{R}_i \right\},
\end{align}
where $\mathbb{B} (\z, \| \z - \p \|_2)$ is the open ball in $\bb{R}^d$ of center $\z$ and radius $\| \z - \p \|_2$. An illustration of this quantity in two dimensions is provided in Fig. \ref{fig:radius_schema}. It is not hard to see that any ball $\mathbb{B} (\z^*, \| \z^* - \p \|_2)$ centered in $\z^*$ and included in $\mathcal{R}_i$ will have its tangent space at $\p$ \textit{coincide} with the tangent of the decision boundary at the same point. %This is shown as follows:

It should further be noted that the definition in Eq. (\ref{eq:radius_ball}) is not symmetric in $i$ and $j$; i.e., $q_{\infdivx{i}{j}} (\p) \neq q_{\infdivx{j}{i}} (\p)$ as the radius of the largest ball one can inscribe in both regions need not be equal. We therefore define the following symmetric quantity $q_{i,j} (\p)$, where the worst-case ball inscribed in any of the two regions is considered:
\begin{align*}
q_{i,j} (\p) = \min \left( q_{\infdivx{i}{j}} (\p), q_{\infdivx{j}{i}} (\p) \right).
\end{align*}
% describes the curvature of the decision boundary locally at $p$ by packing the largest ball included in one of the regions.
This definition describes the curvature of the decision boundary locally at $\p$ by fitting the largest ball included in one of the regions. To measure the global curvature, the worst-case radius is taken over all points on the decision boundary, i.e.,
\begin{align}
q (\B_{i,j} ) & = \inf_{\p \in \B_{i,j}} q_{i,j} (\p), \label{eq:def_RBij} \\
\kappa (\B_{i,j}) & = \frac{1}{q (\B_{i,j} )}. \label{eq:def_kappa}
\end{align}
The curvature $\kappa (\B_{i,j})$ is simply defined as the inverse of the worst-case radius over all points $\p$ on the decision boundary. 

\begin{figure}[ht]
\centering
\includegraphics[scale=0.4]{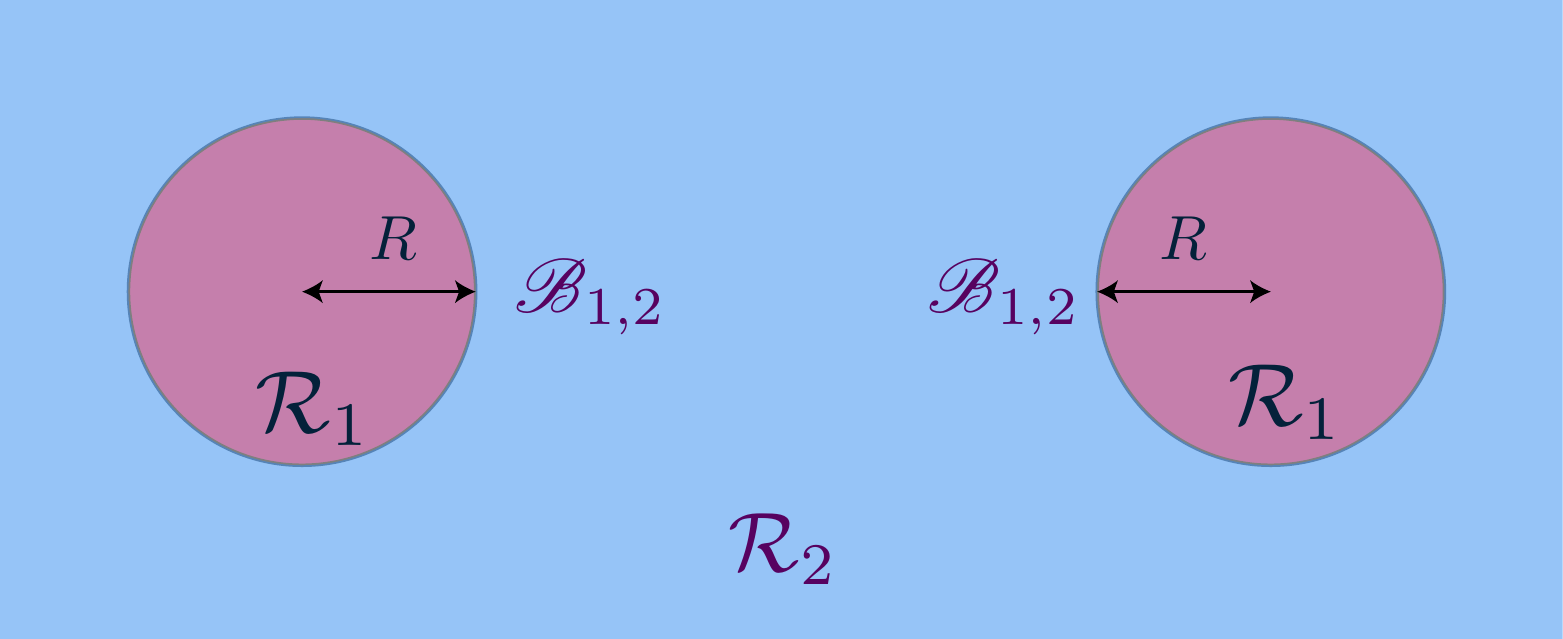}
\caption{\label{fig:sphere_example} Binary classification example where the boundary is a union of two sufficiently distant spheres. In this case, the curvature is $\kappa(\bo_{i,j}) = \nicefrac{1}{R}$, where $R$ is the radius of the circles.}
\end{figure}

In the case of affine classifiers, we have $\kappa (\B_{i,j}) = 0$, as it is possible to inscribe balls of infinite radius inside each region of the space. When the classification boundary is a union of (sufficiently distant) spheres with equal radius $R$ (see Fig. \ref{fig:sphere_example}), the curvature $\kappa (\B_{i,j}) = \nicefrac{1}{R}$. In general, the quantity $\kappa(\B_{i,j})$ provides an intuitive way of describing the nonlinearity of the decision boundary by fitting balls inside the classification regions.

\begin{figure}
\centering
% \subfigure[]{
\includegraphics[scale=0.8]{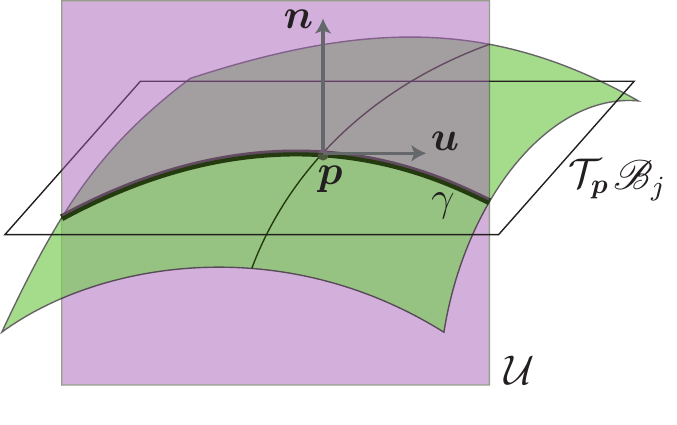}
% }
% \subfigure[]{
% \includegraphics[scale=0.5]{radius_schema_normal_curvature.pdf}
% }
\caption{\label{fig:subadv_nonlin}Normal section of the boundary $\bo_{i,j}$ with respect to plane $\mathcal{U} = \text{span} (\boldsymbol{n}, \boldsymbol{u})$, where $\boldsymbol{n}$ is the normal to the boundary at $\boldsymbol{p}$, and $\boldsymbol{u}$ is an arbitrary in the tangent space $\mathcal{T}_{\boldsymbol{p}}  (\bo_{i,j})$.} % (b) Example where the proposed curvature definition is different from the traditional notion of curvature, as the inverse of the oscullating circle (the dotted curve represents the oscullating circle).}
% \caption{\label{fig:subadv_nonlin}(a) Normal section of the boundary $\bo_{i,j}$ with respect to plane $\mathcal{U}$. (b) Example where the proposed curvature definition is different from the traditional notion of curvature, as the inverse of the oscullating circle (the dotted curve represents the oscullating circle). \red{Change figure on the right if you want to keep it...}}
\end{figure}
% \red{discussion here with respect to normal curvature in latex.}

In the following section, we show a precise characterization of the robustness to semi-random and random noise of nonlinear classifiers in terms of the curvature of the decision boundaries $\kappa(\B_{i,j})$.

\subsection{Robustness to random and semi-random noise}
% with the two classes denoted by $k$ and $\lab(\x_0)$. 
We now establish bounds on the robustness to random and semi-random noise in the binary classification case. Let $\x_0$ be a datapoint classified as $\lab = \lab(\x_0)$. We first study the binary classification problem, where only classes $\lab$ and $k \in \{1, \dots, \numClass\} \backslash \{ \lab \}$ are considered. 
% $k \in \{1, \dots, \numClass\} \backslash \{ \lab \}$.
To simplify the notation, we let $\bo_{k} := \bo_{k, \lab}$ be the decision boundary between classes $k$ and $\lab$. In the case of the binary classification problem where classes $k$ and $\lab$ are considered, the semi-random robustness and adversarial (or worst-case) robustness defined in Eq. (\ref{eq:second_form_robustness}) can be re-written as follows:
\begin{align}
\label{eq:r_s_k}
\begin{split}
\r_\S^k = \argmin_{\r \in \S} \| \r \|_2 \text{ s.t. } f_{k} (\x_0+\r) \geq f_{\lab} (\x_0+\r),  \\
\r^k = \argmin_{\r} \| \r \|_2 \text{ s.t. } f_{k} (\x_0+\r) \geq f_{\lab} (\x_0+\r).
\end{split}
\end{align}
For a randomly chosen subspace, $\|\r_\S^k \|_2$ is the random or semi-random robustness of the classifier, in the setting where only the two classes $k$ and $\lab$ are considered. Likewise, $\|\r^k  \|_2$ denotes the worst-case robustness in this setting. It should be noted that the global quantities $\r_\S^*$ and $\r^*$ are obtained from $\r_\S^k$ and $\r^k$ by taking the vectors with minimum norm over all classes $k$. 

% NOT NECESSARY HERE (AFFINE CLASSIFICATION)
%It should further be noted that, in the case where the classifier is affine, the subspace worst-case perturbation at $\x_0$ can be computed in closed form (for any subspace $\S \subseteq \R^d$): $\|\r_\S^*\|_2=\min_k\|\r_\S^k\|_2$ (\red{put the proof of this last thing before with the union, maybe in DeepFool chapter}), where
%\begin{align*}
%\r_\S^k=\frac{\left|f_{k}-f_{\lab}\right|}{\|\P_\S\w_k-\P_\S\w_{\lab}\|_2^2}(\P_\S\w_{k}-\P_\S\w_{\lab}),
%\end{align*}
%where $\P_\S$ is the projection operator corresponding to $\S$. 
% END NOT NECESSARY HERE

% The superscript $k$ indicates the fact that we consider the binary classification problem with $k$.
% Moreover, since the results we derive are for a fixed datapoint $\x_0$, we remove the explicit dependence on $\x_0$ in our notations of robustness (e.g., $\r_{\S}^k$ instead of $\r_{\S}^k $).

The following result gives upper and lower bounds on the ratio $\frac{\|\r^k_\S\|_2}{\|\r^k\|_2}$ in function of the curvature of the boundary separating class $k$ and $\lab$.
% BINARY MAIN THEOREM
\begin{theorem}
\label{thm:mainThm_constantCurvature_binary}
Let $\S$ be a random $m$-dimensional subspace of $\mathbb{R}^d$. Let $\kappa := \kappa(\bo_k)$. Assuming that the curvature satisfies
\begin{align*}
\kappa \leq \frac{C}{\gatwo \|\r^k\|_2} \frac{m}{d},
\end{align*}
the following inequality holds between the semi-random robustness $\|\r^k_\S\|_2$ and the adversarial robustness $\|\r^k\|_2$:
\begin{equation}
 \left( 1-C_1 \|\r^k\|_2\kappa \gatwo \frac{d}{m} \right) \sqrt{\gaone} \sqrt{\frac{d}{m}} 
\leq \frac{\|\r^k_\S\|_2}{\|\r^k\|_2}
\leq \left(1+C_2\|\r^k\|_2\kappa \gatwo \frac{d}{m} \right) \sqrt{\gatwo} \sqrt{\frac{d}{m}} 
% \label{eq:mainThm_eq}
\end{equation}
with probability larger than $1- 4 \delta$. We recall that $\gaone$ and $\gatwo$ are defined in Eq. (\ref{eq:chap_semirandom_gaone}, \ref{eq:chap_semirandom_gatwo}). The constants are $C = 0.2, C_1 = 0.625, C_2 = 2.25$.
\end{theorem}
The proof can be found in the appendix. This result shows that the bounds relating the robustness to random and semi-random noise to the worst-case robustness can be extended to nonlinear classifiers, provided the curvature of the boundary $\kappa(\bo_k)$ is sufficiently small. % Note that our upper and lower bounds now depend on the curvature $\kappa(\bo_k)$. 
In the case of linear classifiers, we have $\kappa(\B_k) = 0$, and we recover the result for affine classifiers from Theorem \ref{th:linear}.

To extend this result to multi-class classification, special care has to be taken. In particular, if $k$ denotes a class that has no boundary with class $\lab$, we have $\|\r^k\|_2 = \infty$, and the previous curvature condition cannot be satisfied. It is therefore crucial to \textit{exclude} such classes that have no boundary in common with class $\hat{k}$, or more generally, boundaries that are far from class $\lab$. We define the set $A$ of excluded classes $k$ where $\|\r^k\|_2$ is large
\begin{align}
\label{eq:definition}
A = \{ k: \|\r^k\|_2 \geq 1.45 \sqrt{\gatwo} \sqrt{\frac{d}{m}} \| \r^* \|_2 \}.
\end{align}
Note that $A$ is independent of $\S$, and depends only on $d$, $m$ and $\delta$. Moreover, the constants in~(\ref{eq:definition}) were chosen for simplicity of exposition.

Assuming a curvature constraint \textit{only on the close enough classes}, the following result establishes a simplified relation between $\|\r^*_\S\|_2$ and $\|\r^*\|_2$.
\begin{corollary}
\label{corr:nonlinear}
Let $\S$ be a random $m$-dimensional subspace of $\mathbb{R}^d$. Assume that, for all $k \notin A$, we have
\begin{align}
\label{eq:curvature_condition}
\kappa(\bo_k) \|\r^k\|_2 \leq \frac{0.2}{\gatwo} \frac{m}{d}
\end{align}
Then, we have
\begin{equation}
0.875 \sqrt{\gaone} \sqrt{\frac{d}{m}} \|\r^*\|_2 \leq \|\r^*_\S\|_2 \leq 1.45 \sqrt{\gatwo} \sqrt{\frac{d}{m}} \|\r^*\|_2
\end{equation}
with probability larger than $1- 4 (L+2) \delta$.
\end{corollary}

Under the curvature condition in (\ref{eq:curvature_condition}) on the boundaries between $\lab$ and classes in $A^c$, our result shows that the robustness to random and semi-random noise exhibits the same behavior that has been observed earlier for linear classifiers in Theorem \ref{th:linear}. In particular, $\|\r^*_\S\|_2$ is precisely related to the adversarial robustness $\|\r^*\|_2$ by a factor of $\sqrt{\nicefrac{d}{m}}$. In the random regime ($m=1$), this factor becomes $\sqrt{d}$, and shows that in high dimensional classification problems, classifiers with sufficiently flat boundaries are much more robust to random noise than to adversarial noise. More precisely, the addition of a sufficiently small random noise does not change the label of the image, even if the image lies very closely to the decision boundary (i.e., $\|\r^*\|_2$ is small). However, in the semi-random regime where an adversarial perturbation is found on a randomly chosen subspace of dimension $m$, the $\sqrt{\nicefrac{d}{m}}$ factor 
that relates $\|\r^*_\S\|_2$ to $\|\r^*\|_2$ shows that robustness to semi-random noise might not be achieved even if $m$ is chosen to be a tiny fraction of $d$ (e.g., $m = 0.01 d$). In other words, if a classifier is highly vulnerable to adversarial perturbations, then it is also vulnerable to noise that is overwhelmingly random and only mildly adversarial (i.e. worst-case noise sought in a random subspace of low dimensionality $m$).
% might not be enough to achieve robustness even in cases where $m$ is taken as a tiny fraction of $d$ (e.g., $m = 0.01 d$). In other words, if a classifier is highly vulnerable to adversarial perturbations ($m=d$), then it may be also vulnerable to semi-random perturbations even for small values of $m$.

It is important to note that the curvature condition in~(\ref{eq:curvature_condition}) is \textit{not} an assumption on the curvature of the global decision boundary, but rather an assumption on the decision boundaries between pairs of classes. The distinction here is significant, as junction points where two decision boundaries meet might actually have a very large (or infinite) curvature (even in linear classification settings), and the curvature condition in~(\ref{eq:curvature_condition}) typically does not hold for this global curvature definition. We refer to our experimental section for a visualization of this phenomenon.

We finally stress that our results in Theorem \ref{thm:mainThm_constantCurvature_binary} and Corollary \ref{corr:nonlinear} are applicable to \textit{any} classifier, provided the decision boundaries are smooth.
% \footnote{In the case of non-differentiable decision boundary, the curvature $\kappa$ can be infinite. In these cases, one can define an alternative definition of the curvature, where the $\inf$ in Eq. (\ref{eq:def_kappa}) is taken only over differentiable points. Then, our results hold provided the boundary is differentiable at $\x_0 + \r^*(\x_0)$.} 
If we assume prior knowledge on the considered family of classifiers and their decision boundaries (e.g., the decision boundary is a union of spheres in $\R^d$), similar bounds can further be derived under less restrictive curvature conditions (compared to Eq. (\ref{eq:curvature_condition})). % For example, if the decision boundary is known to be an arbitrary union of spheres in $\R^d$, less res

% If we consider specific families of classifiers (e.g., RBF-SVM classifiers), it might be possible to achieve similar bounds under less restrictive conditions (Eq. (\ref{eq:curvature_condition}). For instance, if the decision boundary is known to be an arbitrary union of spheres in $\R^d$, 

% \red{maybe say a few words that the results can be tightened.}

% Another interesting implication of the constraint in Corollary \ref{corr:nonlinear} is that a classifier should have both sufficiently large $\|\r^*\|_2_2$ or sufficiently curved decision boundaries in order to achieve robustness in semi-random regimes. In the next section, we provide estimates of the  quantity of interest $\kappa (\B_k) \| \r^k \|_2$  for different classifiers.

% \subsection{Curvature}

% \subsection{Main result}

% \section{Experiments}

\section{Experiments}
% \label{sect:semiRandom_experiments}
\label{sec:experiments_semirandom}

% \red{Put additional experiments showing the large gap between adversarial and random noise robustness}

\subsection{Experimental results}

We now evaluate the robustness of different image classifiers to random and semi-random perturbations, and assess the accuracy of our bounds on various datasets and state-of-the-art classifiers. Specifically, our theoretical results show that the robustness $\|\r^*_\S(\x)\|_2$ of classifiers satisfying the curvature property precisely behaves as $\sqrt{\nicefrac{d}{m}} \|\r^*(\x)\|_2$. We first check the accuracy of these results in different classification settings.
For a given classifier $f$ and subspace dimension $m$, we define $$\beta(f;m)=\sqrt{\nicefrac{m}{d}}\frac{1}{|\mathscr{D}|}\sum_{\x\in\mathscr{D}}\frac{\|\r^*_\S(\x)\|_2}{\|\r^*(\x)\|_2},$$ where $\S$ is chosen randomly for each sample $\x$ and $\mathscr{D}$ denotes the test set. This quantity provides indication to the accuracy of our $\sqrt{\nicefrac{d}{m}} \|\r^* (\x)\|_2$ estimate of the robustness, and should ideally be equal to $1$ (for sufficiently large $m$).
Since $\beta$ is a random quantity (because of $\S$), we report both its mean and standard deviation for different networks in Table~\ref{tab:deviations}. It should be noted that finding $\|\r^*_\S\|_2$ and $\|\r^*\|_2$ involves solving the optimization problem in (\ref{eq:def_subspace_optimal}). We have used a similar approach to \cite{moosavi2015deepfool} to find subspace minimal perturbations. For each network, we estimate the expectation by averaging $\beta(f; m)$ on 1000 random samples, with $\S$ also chosen randomly for each sample.

\begin{table}
\centering
\begin{tabular}{@{}lcccccc@{}}
\toprule
&\multicolumn{6}{c}{$\nicefrac{m}{d}$}\\
\cmidrule{1-7}
 \textbf{Classifier}         & 1 &\nicefrac{1}{4}&\nicefrac{1}{16}&\nicefrac{1}{36}&\nicefrac{1}{64}&\nicefrac{1}{100}\\
 \midrule
LeNet (MNIST) &$1.00$&$1.00\pm0.06$&$1.01\pm0.12$&$1.03\pm0.20$&$1.01\pm0.26$&$1.05\pm0.34$\\
LeNet (CIFAR-10) &$1.00$&$1.01\pm0.03$&$1.02\pm0.07$&$1.04\pm0.10$&$1.06\pm0.14$&$1.10\pm0.19$\\
VGG-F (ImageNet)&$1.00$&$1.00\pm0.01$&$1.02\pm0.02$&$1.03\pm0.04$&$1.03\pm0.05$&$1.04\pm0.06$ \\
VGG-19 (ImageNet)&$1.00$&$1.00\pm0.01$&$1.02\pm0.03$&$1.02\pm0.05$&$1.03\pm0.06$&$1.04\pm0.08$\\
\bottomrule
\end{tabular}
\caption{\label{tab:deviations} $\beta(f; m)$ for different classifiers $f$ and different subspace dimensions $m$. The VGG-F and VGG-19 are respectively introduced in \citep{chatfield2014, simonyan2014very}.}
\end{table}
% Theorem \ref{thm:mainThm_constantCurvature_binary}
Observe that $\beta$ is suprisingly close to 1, even when $m$ is a small fraction of $d$. This shows that our quantitative analysis provide very accurate estimates of the robustness to semi-random noise. We visualize the robustness to random noise, semi-random noise (with $m = 10$) and worst-case perturbations on a sample image in Fig. \ref{fig:sequence}. While random noise is clearly perceptible due to the $\sqrt{d} \approx 400$ factor, semi-random noise becomes much less perceptible even with a relatively small value of $m = 10
$, thanks to the $\nicefrac{1}{\sqrt{m}}$ factor that attenuates the required noise to misclassify the datapoint. It should be noted that the robustness of neural networks to adversarial perturbations has previously been observed empirically in \cite{szegedy2013intriguing}, but we provide here a quantitative and generic explanation for this phenomenon.

\begin{figure}[ht]
\centering
\subfigure[]{
\includegraphics[width=0.2\textwidth]{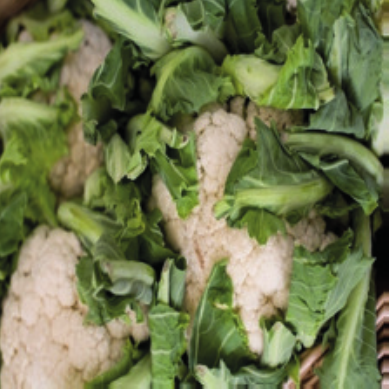}
% \caption{\label{fig:seq0}}
}
\subfigure[]{
\includegraphics[width=0.2\textwidth]{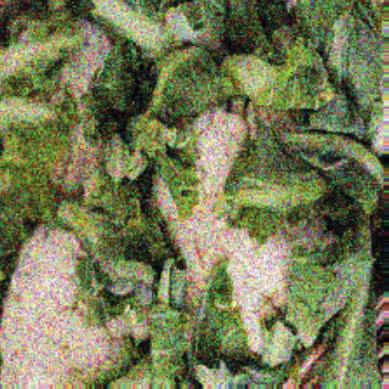}
% \caption{\label{fig:seq1}}
}
\subfigure[]{
\includegraphics[width=0.2\textwidth]{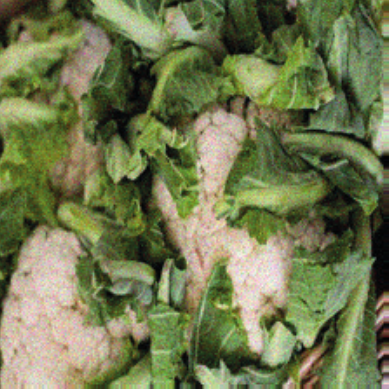}
% \caption{\label{fig:seq2}}
}
\subfigure[]{
\includegraphics[width=0.2\textwidth]{seq0.eps}
% \caption{\label{fig:seq3}}
}
\caption{\label{fig:sequence} (a) Original image classified as ``Cauliflower''. Fooling perturbations for VGG-F network: (b) Random noise, (c) Semi-random perturbation with $m=10$, (d) Worst-case perturbation, all wrongly classified as ``Artichoke''.}
\end{figure}

% \red{START: Here change depending on the notion of curvature finally adopted}

The high accuracy of our bounds for different state-of-the-art classifiers, and different datasets suggest that the decision boundaries of these classifiers have limited curvature $\kappa(\B_k)$, as this is a key assumption of our theoretical findings. To support the validity of this curvature hypothesis in practice, we visualize two-dimensional sections of the classifiers' boundary in Fig.~\ref{fig:visualize} in three different settings. Note that we have opted here for a visualization strategy rather than the numerical estimation of $\kappa(\B)$, as the latter quantity is difficult to approximate in practice in high dimensional problems. In Fig.~\ref{fig:visualize}, $\x_0$ is chosen randomly from the test set for each data set, and the decision boundaries are shown in the plane spanned by $\r^*$ and $\r^*_\S$, where $\S$ is a random \textit{direction} (i.e., $m=1$). 
%Specifically, the following procedure is applied to sample from the boundary $\bo_k$:
%\begin{enumerate}
%\item Choose a random direction, and estimate $\r^*_\S$ using Algorithm \ref{alg:multiclass_semirandom}. Compute also the worst-case perturbation $\r^*$.
%\item For each discretized value $\alpha_i \in [-T,T]$, do:
%\begin{enumerate}
%\item Define the datapoint $\x_i = \x_0 + \r^* + \alpha_i (\r_{\S}^* - \r^*)$. % \alpha_i \r^* + (1-\alpha_i) \r_{\S}^*
%\item Project the datapoint $\x_i$ onto the decision boundary using Algorithm \ref{alg:multiclass_semirandom}: find the minimal subspace perturbation $\r_{i}^{\text{2d}} \in \text{span} (\r^*, \r_{\S}^*)$ such that $\x_i + \r_{i}^{\text{2d}}$ belongs to the decision boundary. Plot the projected point. % \red{Exactly, how do you do that (is a threshold used?)}
%\end{enumerate}
%\end{enumerate}
% \red{Ask to Seyed: is this correct? How do you color things in red and green? Do you do linear combinations from the first point, or you take the current point, go one step in the direction of $\x_0$ and reproject? Do you constrain $\alpha \in [0,1]$?}
% Specifically, the decision boundaries are sampled using a variation of the method in ~\citep{moosavi2016} in this two-dimensional subspace. \red{We skip the details of this method for the sake of brevity.} \red{Explain in details here...}
Different colors on the boundary correspond to boundaries with different classes. It can be observed that the curvature of the boundary is very small except at ``junction'' points where the boundary of two different classes intersect. Our curvature assumption in Eq. (\ref{eq:curvature_condition}), which only assumes a bound on the curvature of the decision boundary between pairs of classes ${\lab(\x_0)}$ and $k$ (but not on the \textit{global} decision boundary that contains junctions with high curvature) is therefore adequate to the decision boundaries of state-of-the-art classifiers according to Fig.~\ref{fig:visualize}. Interestingly, the assumption in Corollary \ref{corr:nonlinear} is satisfied by taking $\kappa$ to be an empirical estimate of the curvature of the planar curves in Fig. \ref{fig:visualize} (a) for the dimension of the subspace being a \textit{very} small fraction of $d$; e.g., $m = 10^{-3} d$. While not reflecting the curvature $\kappa(\B_k)$ that drives the assumption of our theoretical analysis, this result still seems to suggest that the curvature assumption holds in practice, and that the curvature of such classifiers is therefore very small. It should be noted that a related empirical observation was made in~\cite{goodfellow2014}; our work however provides a precise quantitative analysis on the relation between the curvature and the robustness in the semi-random noise regime.

% \red{END: Here change depending on the notion of curvature finally adopted}

\begin{figure}[ht]
\centering
\subfigure[VGG-F (ImageNet)] {
\includegraphics[width=0.3\textwidth]{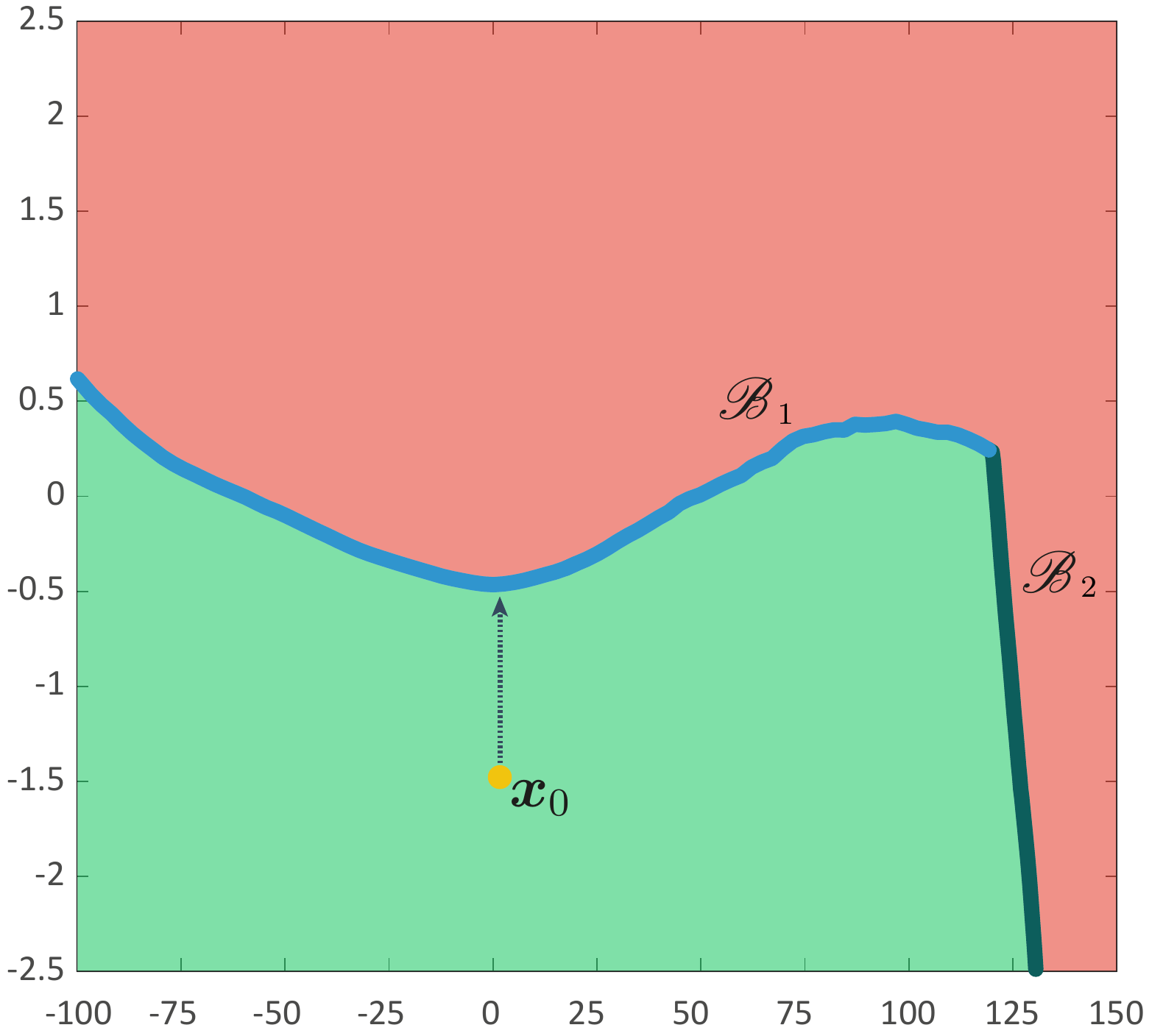}
}
\subfigure[LeNet (CIFAR)]{
\includegraphics[width=0.3\textwidth]{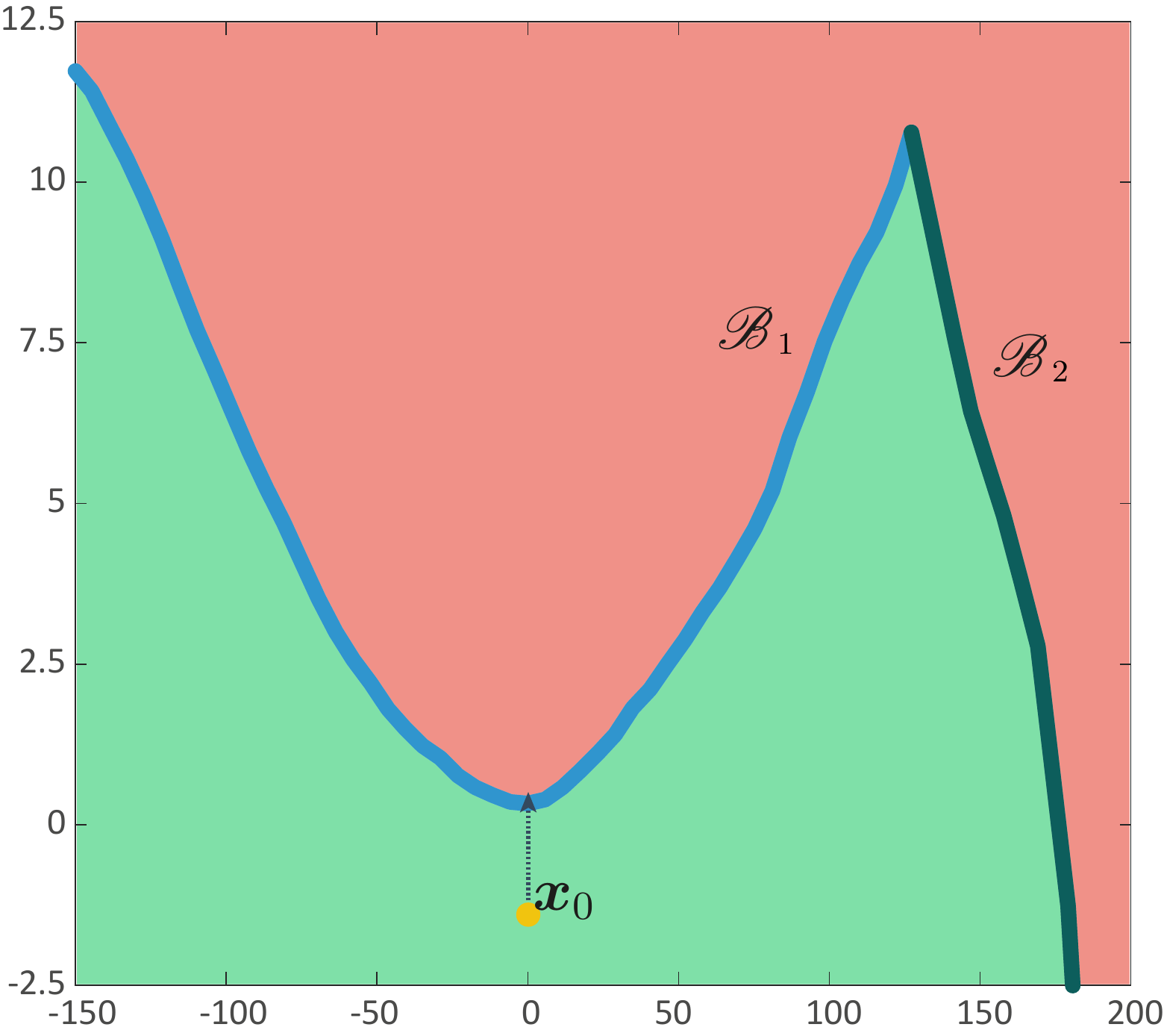}
}
\subfigure[LeNet (MNIST)]{
\includegraphics[width=0.3\textwidth]{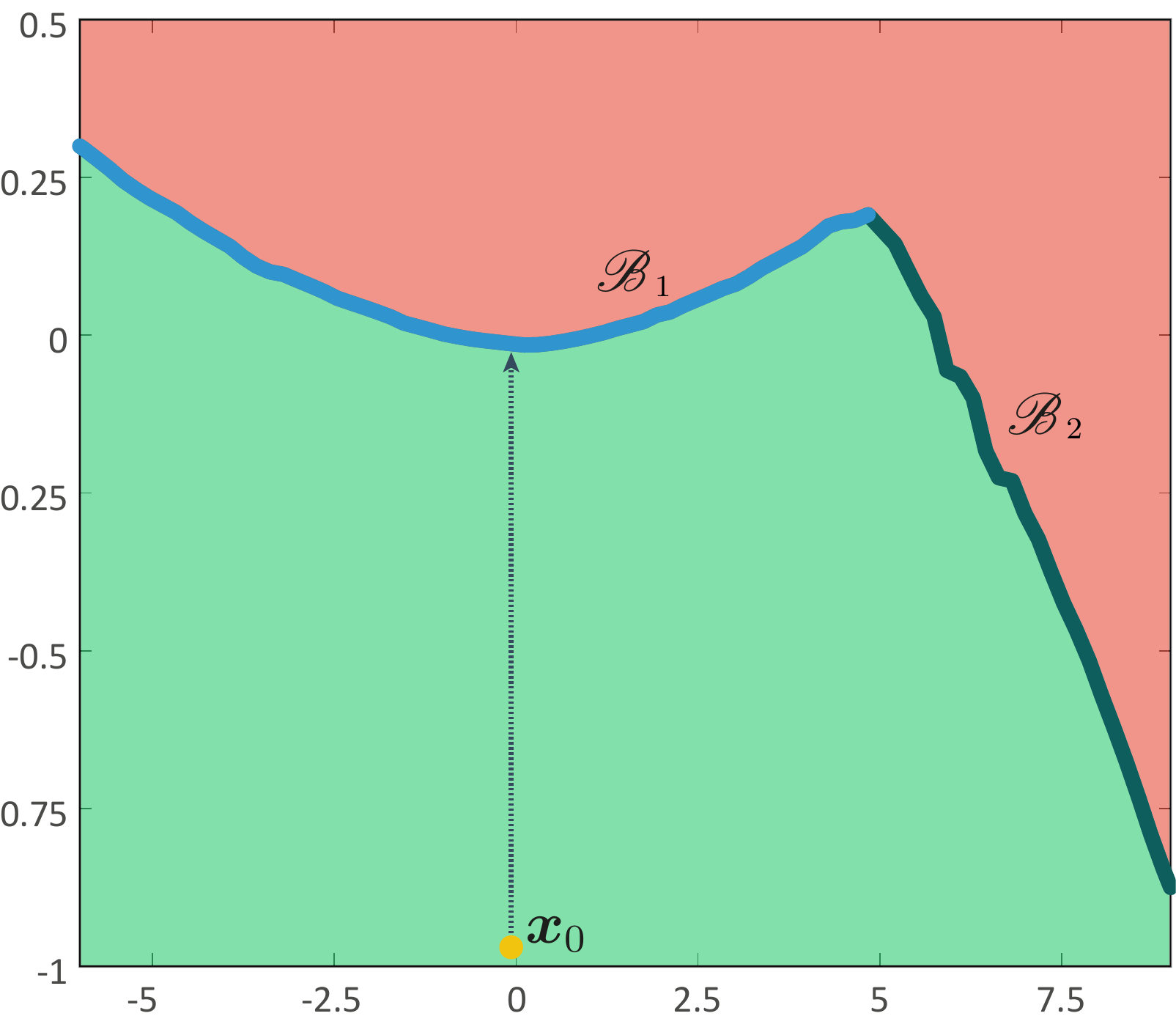}
}
\caption{\label{fig:visualize}Boundaries of three classifiers near randomly chosen samples. Axes are normalized by the corresponding $\|\r^*\|_2$ since our assumption in the theoretical bound (Corollary \ref{corr:nonlinear}) depends on the product of $\|\r^*\|_2 \kappa$. Note the difference in range between $x$ and $y$ axes. Note also that the range of horizontal axis in (c) is much smaller than the other two, hence the illustrated boundary is more curved.}
\end{figure}

We now show a simple demonstration of the vulnerability of classifiers to semi-random noise in Fig.~\ref{fig:nips}, where a structured message is hidden in the image and causes data misclassification. Specifically, we consider $\S$ to be the span of random translated and scaled versions of words ``NIPS'', ``SPAIN'' and ``2016'' in an image, such that $\lfloor\nicefrac{d}{m}\rfloor=228$. The resulting perturbations in the subspace are therefore linear combinations of these words with different intensities.\footnote{This example departs somehow from the theoretical framework of this paper, where \textit{random} subspaces were considered. However, this empirical example suggests that the theoretical findings in this paper seem to approximately hold when the subspace $\S$ have statistics that are close to a random subspace.} The perturbed image $\x_0 + \r_{\S}^*$ shown in Fig.~\ref{fig:nips} (c) is clearly indistinguishable from Fig.~\ref{fig:nips} (a). This shows that imperceptibly small structured messages can be added to an image causing data misclassification. % We believe that this example, where an imperceptibly small structured message is added to the image to change the estimated label of the classifier, might possibly be extended and lead to automatic watermarking and steganography applications, where data is hidden in an image for tracking or to send hidden messages.

% to hide a structured message inside the image that fools state-of-the-art classifiers might possibly 

\begin{figure}[ht]
\centering
\subfigure[Image of a ``Potflower'']{
\includegraphics[width=0.3\textwidth]{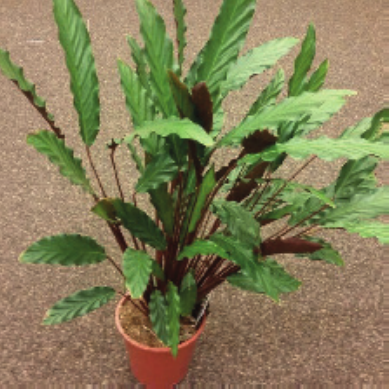}
}
\subfigure[Structured perturbation containing random placement of words ``NIPS'', ``2016'', and ``SPAIN'']{
\includegraphics[width=0.3\textwidth]{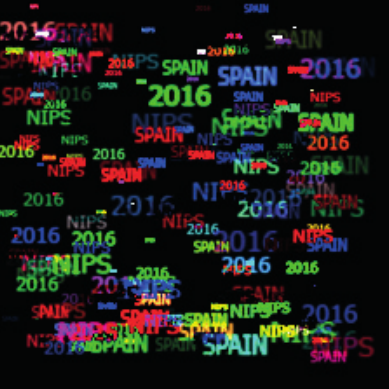}
}
\subfigure[Classified as ``Pineapple'']{
\includegraphics[width=0.3\textwidth]{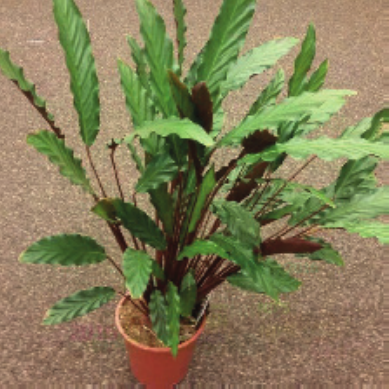}
}
\caption{A fooling hidden message, $\S$ consists of linear combinations of random words.}
\label{fig:nips}
\end{figure}
\section{Conclusion}

In this work, we precisely characterized the robustness of classifiers in a novel semi-random noise regime that generalizes the random noise regime. Specifically, our bounds relate the robustness in this regime to the robustness to adversarial perturbations. Our bounds depend on the \textit{curvature} of the decision boundary, the data dimension, and the dimension of the subspace to which the perturbation belongs. Our results show, in particular, that when the decision boundary has a small curvature, classifiers are robust to random noise in high dimensional classification problems (even if the robustness to adversarial perturbations is relatively small). Moreover, for semi-random noise that is mostly random and only mildly adversarial (i.e., the subspace dimension is small), our results show that state-of-the-art classifiers remain vulnerable to such perturbations. To improve the robustness to semi-random noise, our analysis encourages to impose geometric constraints on the curvature of the decision boundary, as we have shown the existence of an intimate relation between the robustness of classifiers and the curvature of the decision boundary.
\subsubsection*{Acknowledgments}
We would like to thank the anonymous reviewers for their helpful comments. We thank Omar Fawzi and Louis Merlin for the fruitful discussions. We also gratefully acknowledge the support of NVIDIA Corporation with the donation of the Tesla K40 GPU used for this research. This work has been partly supported by the Hasler Foundation, Switzerland, in the framework of the CORA project.
%\section*{References}
\medskip
\small{
\bibliographystyle{apalike}
\bibliography{bibliography.bib}
}
\section*{Appendix}
\label{app:semirandom}
\renewcommand{\thesubsection}{A.\arabic{subsection}}
\subsection{Proof of Theorem \ref{th:main_result_linear} (affine classifiers)}
\begin{lemma}[\cite{dasgupta2003elementary}]
\label{lem:jl_l}
Let $Y$ be a point chosen uniformly at random from the surface of the $d$-dimensional sphere $\mathbb{S}^{d-1}$. Let the vector $Z$ be the projection of $Y$ onto its first $m$ coordinates, with $m < d$. Then,
\begin{enumerate}
\item If $\beta < 1$, then
\begin{align}
\Pbb \left( \| Z \|_2^2 \leq \frac{\beta m}{d}\right) \leq \beta^{m/2} \left( 1 + \frac{(1 - \beta) m}{(d-m)} \right)^{(d-m)/2} \leq \exp\left( \frac{m}{2} (1 - \beta + \ln \beta) \right).
\end{align}
\item If $\beta > 1$, then
\begin{align}
\Pbb \left( \| Z \|_2^2 \geq \frac{\beta m}{d} \right) \leq \beta^{m/2} \left( 1 + \frac{(1 - \beta) m}{(d-m)}\right)^{(d-m)/2} \leq \exp\left( \frac{m}{2} \left( 1 - \beta + \ln \beta \right) \right).
\end{align}
\end{enumerate}
\end{lemma}

\begin{lemma}
\label{thm:jl_ours}
Let $\v$ be a random vector uniformly drawn from the unit sphere  $\mathbb{S}^{d-1}$, and $\P_m$ be the projection matrix onto the first $m$ coordinates. Then,
\begin{align}
\Pbb\left( \beta_1(\delta, m) \frac{m}{d} \leq \| \P_m \v \|_2^2 \leq \beta_2(\delta, m) \frac{m}{d} \right) \geq 1 - 2\delta,
\end{align}
with $\beta_1(\delta, m) = \max((1/e) \delta^{2/m}, 1-\sqrt{2(1-\delta^{2/m})}$, and $\beta_2(\delta, m) = 1 + 2 \sqrt{\frac{ \ln(1/\delta)}{m}} + \frac{2 \ln(1/\delta)}{m}$.
\end{lemma}
\begin{proof}
Note first that the upper bound of Lemma \ref{lem:jl_l} can be bounded as follows:
\begin{align}
\beta^{m/2} \left( 1 + \frac{(1-\beta)m}{d-m} \right)^{(d-m)/2} \leq \beta^{m/2} \exp \left(\frac{(1-\beta) m}{2} \right),
\end{align}
using $1+x \leq \exp(x)$. We find $\beta$ such that $\beta^{m/2} \exp\left( \frac{(1-\beta)m}{2} \right) \leq \delta$, or equivalently,
$\beta \exp\left( 1-\beta \right) \leq \delta^{2/m}$.
It is easy to see that when $\beta = \frac{1}{e} \delta^{2/m}$, the inequality holds. Note however that $\frac{1}{e} \delta^{2/m}$ does not converge to $1$ as $m \rightarrow \infty$. We therefore need to derive a tighter bound for this regime. Using the inequality $\beta \exp(1-\beta) \leq 1 - \frac{1}{2} (1-\beta)^2$ for $0 \leq \beta \leq 1$, it follows that the inequality $\beta \exp( 1 - \beta ) \leq \delta^{2/m}$ holds for $\beta = 1 - \sqrt{2(1-\delta^{2/m})}$. In this case, we have $1 - \sqrt{2(1-\delta^{2/m})} \rightarrow 1$, as $m \rightarrow \infty$. We take our lower bound to be the max of both derived bounds (the latter is more appropriate for large $m$, whereas the former is tighter for small $m$).

For $\beta_2$, note that the requirement $\beta \exp(1-\beta) \leq \delta^{2/m}$ is equivalent to $-\ln(\beta) + (\beta-1) \geq \frac{2}{m} \ln(1/\delta)$. By setting $\beta = \beta_2(\delta, m)$, this condition is equivalent to $2 \sqrt{\frac{\ln(1/\delta)}{m}} - \ln(\beta_2(\delta,m)) \geq 0$, or equivalently, $2 z - \ln(1+2z+2z^2) \geq 0$, with $z = \sqrt{\frac{\ln(1/\delta)}{m}}$. The function $z \mapsto 2 z - \ln(1+2z+2z^2) \geq 0$ is positive on $\mathbb{R}^+$. Hence, $\beta_2(\delta,m)$ satisfies $\beta \exp( 1 - \beta ) \leq \delta^{2/m}$, which concludes the proof.
% Using Lemma \ref{lem:jl_l}, we have for any
\end{proof}

We now prove our main theorem that we recall as follows:
\begin{reptheorem}{th:linear}
Let $\S$ be a random $m$-dimensional subspace of $\mathbb{R}^d$. The following inequalities hold between the norms of semi-random perturbation  $\r_{\S}^*$ and the worst-case perturbation $\r^*$. Let $\gaone = \frac{1}{\beta_2(m, \delta)}$, and $\gatwo = \frac{1}{\beta_1(m, \delta)}$.
\begin{equation}
\gaone \frac{d}{m} \|\r^*\|_2^2 \leq
\|\r_\S^*\|_2^2\leq
\gatwo \frac{d}{m} \|\r^*\|_2^2,
\end{equation}
with probability exceeding $1 - 2 (L+1) \delta$.
% \label{lem:linear}
\end{reptheorem}

\begin{proof}
For the linear case, $\r^*$ and $\r_{\S}^*$ can be computed in closed form. We recall that, for any subspace $\S$, we have
\begin{equation}
\r_\S^k=\frac{\left|f_{k}(\x_0)-f_{\lab(\x_0)}(\x_0)\right|}{\|\P_\S\w_k-\P_\S\w_{\lab(\x_0)}\|_2^2}(\P_\S\w_{k}-\P_\S\w_{\lab(\x_0)}),
\end{equation}
where $\r_\S^k$ was defined in Eq. (\ref{eq:r_s_k}). In particular, when $\S = \mathbb{R}^d$, we have
\begin{equation}
\r^k=\frac{\left|f_{k}(\x_0)-f_{\lab(\x_0)}(\x_0)\right|}{\|\w_k-\w_{\lab(\x_0)}\|_2^2}(\w_{k}-\w_{\lab(\x_0)}).
\end{equation}
Let $k \neq \lab(\x_0)$. Define, for the sake of readability
\begin{equation*}
\begin{split}
&f^k=\left|f_{k}(\x_0)-f_{\lab(\x_0)}(\x_0)\right|,\\
&\z^k=\w_k-\w_{\lab(\x_0)}.
\end{split}
\end{equation*}
Note that
\begin{align}
\frac{\| \r^k \|_2^2}{\| \r_{\S}^k \|_2^2} = \frac{\| \P_{\S} \z^k \|_2^2}{\| \z^k \|_2^2}.
\end{align}
The projection of a fixed vector in $\mathbb{S}^{d-1}$ onto a random $m$ dimensional subspace is equivalent (up to a unitary transformation $\mathbf{U}$) to the projection of a random vector uniformly sampled from $\mathbb{S}^{d-1}$ into a fixed subspace. Let $\P_m$ be the projection onto the first $m$ coordinates. We have
\begin{align}
\| \P_{\S} \z^k \|_2^2 = \| \mathbf{U}^T \P_m \mathbf{U} \z^k \|_2^2 = \| \P_m \mathbf{U} \z^k \|_2,
\end{align}
Hence, we have
\begin{align}
\frac{\| \P_{\S} \z^k \|_2^2}{\| \z^k \|_2^2} = \| \P_m \y \|_2^2,
\end{align}
where $\y$ is a random vector distributed uniformly in the unit sphere $\mathbb{S}^{d-1}$. We apply Lemma \ref{thm:jl_ours}, and obtain
\begin{align}
\Pbb \left( \beta_1(m, \delta) \frac{m}{d} \leq \| \P_m \y \|_2^2 \leq \beta_2(m, \delta) \frac{m}{d} \right) \geq 1 - 2 \delta.
\end{align}
Hence,
\begin{align}
\Pbb \left\{ \frac{1}{\beta_2(m, \delta)} \frac{d}{m} \leq \frac{\| \r_{\S}^k \|_2^2 }{ \| \r^k \|_2^2} \leq \frac{1}{\beta_1(m, \delta)} \frac{d}{m} \right\} \geq 1 - 2 \delta.
\end{align}
Using the multi-class extension in Lemma \ref{lem:single_class_multiclass_linear}, we conclude that
\begin{align}
\Pbb \left\{ \gaone \frac{d}{m} \leq \frac{\| \r_{\S}^* \|_2^2 }{ \| \r^* \|_2^2} \leq \gatwo \frac{d}{m} \right\} \geq 1 - 2 (L + 1) \delta.
\end{align}
\end{proof}

\begin{lemma}[Binary case to multiclass]
% Define
%\begin{align}
%\r_\S^k & = \argmin_{\r \in \S} \| \r \|_2 \text{ s.t. } f_{k} (\x_0+\r) \geq f_{\lab(\x_0)} (\x_0+\r), \\
%\r^k & = \argmin_{\r} \| \r \|_2 \text{ s.t. } f_{k} (\x_0+\r) \geq f_{\lab(\x_0)} (\x_0+\r).
%\end{align}
Assume that, for all $k \in \{1, \dots, \numClass]\} \backslash \{ \lab(\x_0) \}$
\begin{align}
\Pbb\left( l \leq \frac{\| \r_\S^k \|_2}{\| \r^k \|_2} \leq u \right) \geq 1 - \delta.
\end{align}
Then, we have
\begin{align}
\Pbb\left( l \leq \frac{\| \r_\S^* \|_2}{\| \r^* \|_2} \leq u \right) \geq 1 - (\numClass + 1) \delta.
\end{align}
\label{lem:single_class_multiclass_linear}
\end{lemma}
\begin{proof}
% \red{TO CHANGE HERE TO REPLACE 2L by L+1}
Let $p := \arg\min_{i} \| \r^i \|_2$. Note that we have $\Pbb\left( \frac{\| \r_{\S}^* \|_2 }{ \| \r^* \|_2} \geq u \right) \leq \Pbb \left( \frac{\| \r_{\S}^p \|_2 }{ \| \r^p \|_2} \geq u \right) \leq \delta$. Moreover, we use a union bound to bound the the other bad event probability:
\begin{align}
\Pbb\left( \frac{\| \r_{\S}^* \|_2 }{ \| \r^* \|_2 } \leq l \right) & \leq \Pbb\left( \bigcup_{k}  \left\{ \frac{\| \r_{\S}^k \|_2 }{ \| \r^k \|_2 } \leq l \right\} \right) \leq \numClass \delta, \\
% \Pbb\left( \frac{\| \r_{\S}^* \|_2 }{ \| \r^* \|_2 } \geq u \right) & \leq \Pbb\left( \bigcup_{k}  \left\{ \frac{\| \r_{\S}^k \|_2 }{ \| \r^k \|_2 } \geq u \right\} \right) \leq \numClass \delta.
\end{align}
We conclude by using the fact that
\begin{align}
\Pbb\left( l \leq \frac{\| \r_\S^* \|_2}{\| \r^* \|_2} \leq u \right)= 1 - \Pbb\left( \frac{\| \r_{\S}^* \|_2 }{ \| \r^* \|_2 } \leq l \right) - \Pbb\left( \frac{\| \r_{\S}^* \|_2 }{ \| \r^* \|_2 } \geq u \right).
\end{align}
\end{proof}

\subsection{Proof of Theorem \ref{thm:mainThm_constantCurvature_binary} and Corollary \ref{corr:nonlinear} (nonlinear classifiers)}

First, we present an important geometric lemma and then use it to bound $\|\r^*_\S\|_2$. For the sake of the general readability of the section, some auxiliary results are given in Section \ref{subsec:useful}.

In the following result, we show that, when the curvature of a planar curve is constant and sufficiently small, the distance between a point $\x$ and the curve at a specific direction $\theta$ is well approximated by the distance between $\x$ and a straight line (see Fig.~\ref{fig:bounding_2d} for an illustration).
\begin{lemma}
\label{lem:geometric}
% \red{check this lemma.}
% \red{I don't think the first assumption is needed for the lower bound.}
Let $\gamma$ be a planar curve of constant curvature $\kappa$. We denote by $r$ the distance between a point $\x$ and the curve $\gamma$. Denote moreover by $\mathcal{T}$ the tangent to $\gamma$ at the closest point to $\x$ (see Fig.~\ref{fig:bounding_2d}). Let $\theta$ be the angle between $\u$ and $\boldsymbol{v}$ as depicted in Fig.~\ref{fig:bounding_2d}. We assume that $r \kappa < 1$. We have
\begin{equation}
-C_1r\kappa\tan^2(\theta)
\leq
\frac{\|\x_\gamma-\x\|_2}{\|\u\|_2}-1
\end{equation}
Moreover, if
% Let $r$ be the distance between a point $\x$ and a curve $\gamma$ and  $\mathcal{T}$ is tangent to $\gamma$ at the closest point to $\x$.  Also, let curve $\gamma$ have some constant curvature $\kappa$.
% Let $\theta$ be the angle between $\u$ and $\boldsymbol{v}$ as depicted in Figure \ref{fig:subadv_2d}. If
$$\tan^2(\theta)\leq \frac{0.2}{r\kappa},$$
then, the following upper bound holds
\begin{equation}
\label{eq:upper_bound_lemma}
\frac{\|\x_\gamma-\x\|_2}{\|\u\|_2}-1
\leq
C_2r\kappa\tan^2(\theta).
\end{equation}
We can set $C_1=0.625$ and $C_2=2.25$.
\end{lemma}
\begin{figure}
\center
\includegraphics[scale=1]{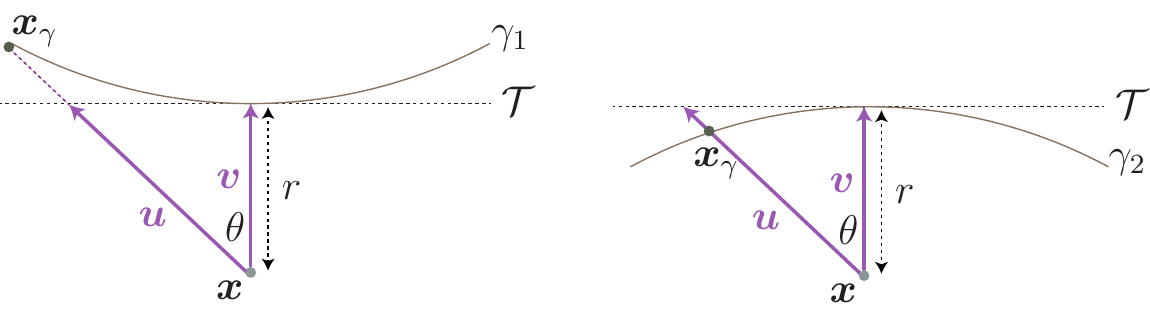}
\caption{Bounding $\|\x_\gamma-\x\|_2$ in terms of $\kappa$.}
\label{fig:bounding_2d}
\end{figure}
\begin{proof}[Proof of upper bound]
We consider two distinct cases for the curve $\gamma$. In the case where $\gamma$ is concave-shaped (Fig. \ref{fig:bounding_2d}, right figure), we have $$\frac{\| \x_{\gamma} - \x \|_2}{\| \u \|_2} \leq 1,$$ and the upper bound in Eq. (\ref{eq:upper_bound_lemma}) directly holds. We therefore focus on the case where $\gamma$ is convex-shaped as illustrated in the left figure of Fig. \ref{fig:bounding_2d}.
% For the upper bound, the worst case is when $\gamma$ coincides with $\gamma_1$.
Define $R:=\nicefrac{1}{\kappa}$, one can write using simple geometric inspection
\begin{equation}
R^2=\sin({\theta})r'^2+(R+r-r'\cos({\theta}))^2,
\end{equation}
where $r'=\|\x_{\gamma}-\x\|_2$. The discriminant of the second order equation (with variable $r'$) is equal to 
\begin{align*}
\Delta = 4 \left( (R+r)^2 \cos^2(\theta) - (2 r R + r^2) \right).
\end{align*}
We have $\Delta \geq 0$ as $\theta$ satisfies the two assumptions $\tan^2 (\theta) \leq 0.2 R / r$ and $r / R < 1$. The smallest solution of this second order equation is given as follows % Solving this second order equation, accordingly to the geometry of the problem, gives
\begin{equation}
r'=(R+r)\cos(\theta)-\sqrt{(R+r)^2\cos^2(\theta)-2Rr-r^2}.
\end{equation}
Using some simple algebraic manipulations, we obtain
\begin{equation}
r'= \frac{r}{\cos({\theta})}\left(\left(\frac{R}{r}+1\right)\cos^2({\theta})-\frac{R}{r}\cos^2({\theta})\sqrt{1-\tan^2({\theta})\frac{2Rr+r^2}{R^2}}\right).
\end{equation}
Using the inequality in Lemma $\ref{lem:ineq}$ together with the two assumptions, we get
\begin{equation}
\begin{split}
r'\leq\frac{r}{\cos({\theta})}
	\Bigg(
		\cos^2({\theta})&+\frac{R}{r}\cos^2({\theta})\tan^2({\theta})\left(\frac{2Rr+r^2}{2R^2}\right)
		\\&+\frac{R}{r}\cos^2({\theta})\tan^4({\theta})\left(\frac{2Rr+r^2}{2R^2}\right)^2
	\Bigg).
\end{split}
\end{equation}
With simple trigonometric identities, the above expression can be simplified to
\begin{equation}
r'\leq\frac{r}{\cos({\theta})}
	\left(
		1+\frac{r}{R}\left(\frac{\sin^2({\theta})}{2}+\frac{\sin^4({\theta})}{\cos^2({\theta})}\left(1+\frac{r}{2R}\right)^2\right)
	\right).
\end{equation}
We expand this quantity, and obtain
\begin{equation}
r'\leq\frac{r}{\cos({\theta})}
	\left(
		1+\left(\frac{\sin^2({\theta})}{2}+\frac{\sin^4({\theta})}{\cos^2({\theta})}\right)\frac{r}{R}
		+\frac{\sin^4({\theta})}{\cos^2({\theta})}\frac{r^2}{R^2}
		+\frac{\sin^4({\theta})}{4\cos^2({\theta})}\frac{r^3}{R^3}
	\right).
\end{equation}
Since $\sin^2({\theta})\tan^2({\theta})=\tan^2({\theta})-\sin^2({\theta})$, we have
\begin{equation}
r'\leq\frac{r}{\cos({\theta})}
	\left(
		1+\tan^2({\theta})\left(\frac{r}{R}
		+\frac{r^2}{R^2}
		+\frac{r^3}{4R^3}\right)
	\right).
\end{equation}
According to the assumptions $r/R<1$, therefore
\begin{equation}
r'\leq\frac{r}{\cos({\theta})}
	\left(
		1+2.25\tan^2({\theta})\frac{r}{R}
	\right).
\end{equation}
Since $r/\cos(\theta)=\|\u\|_2$, one can finally conclude on the upper bound
\begin{equation}
\frac{\|\x_\gamma-\x\|_2}{\|\u\|_2}-1
\leq 2.25r\kappa\tan^2({\theta}).
\end{equation}
\end{proof}

\begin{proof}[Proof of lower bound]
When the curve is convex shaped (Fig. \ref{fig:bounding_2d} left), we have $\| \x_{\gamma} - \x \|_2 \geq \| \u \|_2$, and the desired lower bound holds. We focus therefore on the case where $\gamma$ has a concave shape, and coincides with with $\gamma_2$ (see Fig. \ref{fig:bounding_2d} right).
The following equation holds using simple geometric arguments
\begin{equation}
R^2=\sin({\theta})r'^2+(R-r+r'\cos({\theta}))^2.
\end{equation}
where $r' = \| \x_{\gamma} - \x \|_2$.
Solving this second order equation gives
\begin{equation}
r' = - (R - r) \cos(\theta) + \sqrt{(R-r)^2 \cos^2(\theta) - r^2 + 2 R r}.
\label{eq:lowerbound_secondOrder}
\end{equation}
After some algebraic manipulations, we get
\begin{equation}
r'= \frac{r}{\cos({\theta})}\left(-\left(\frac{R}{r}-1\right)\cos^2({\theta})+\frac{R}{r}\cos^2({\theta})\sqrt{1+\tan^2({\theta})\frac{2Rr-r^2}{R^2}}\right).
\end{equation}
Using the inequality in Lemma \ref{lem:ineq2}, together with the fact that $r \kappa < 1$, we obtain
\begin{equation}
\begin{split}
r'\geq\frac{r}{\cos({\theta})}
	\Bigg(
		\cos^2({\theta})&+\frac{R}{r}\cos^2({\theta})\tan^2({\theta})\left(\frac{2Rr-r^2}{2R^2}\right)
		\\&-\frac{R}{r}\frac{\cos^2({\theta})\tan^4({\theta})}{2}\left(\frac{2Rr-r^2}{2R^2}\right)^2
	\Bigg).
\end{split}
\end{equation}
Using simple trigonometric identities, the above expression is simplified to
\begin{equation}
r'\geq\frac{r}{\cos({\theta})}
	\left(
		1+\frac{r}{R}\left(-\frac{\sin^2({\theta})}{2}-\frac{\sin^4({\theta})}{2\cos^2({\theta})}\left(1-\frac{r}{2R}\right)^2\right)
	\right).
\end{equation}
When expanding it, we obtain
\begin{equation}
r'\geq\frac{r}{\cos({\theta})}
	\left(
		1-\left(\frac{\sin^2({\theta})}{2}+\frac{\sin^4({\theta})}{2\cos^2({\theta})}\right)\frac{r}{R}
		+\frac{\sin^4({\theta})}{2\cos^2({\theta})}\frac{r^2}{R^2}
		-\frac{\sin^4({\theta})}{8\cos^2({\theta})}\frac{r^3}{R^3}
	\right).
\end{equation}
Since $\sin^2({\theta})\tan^2({\theta})=\tan^2({\theta})-\sin^2({\theta})$, we have
\begin{equation}
r'\geq\frac{r}{\cos({\theta})}
	\left(
		1-\tan^2({\theta})\left(\frac{r}{2R}
		+\frac{r^3}{8R^3}\right)
	\right).
\end{equation}
Using again the assumption $r/R<1$, we obtain
\begin{equation}
r'\geq\frac{r}{\cos({\theta})}
	\left(
		1-0.625\tan^2({\theta})\frac{r}{R}
	\right).
\end{equation}
Since $r/\cos(\theta)=\|\u\|_2$, one can rewrite it as
\begin{equation}
\frac{\|\x_\gamma-\x\|_2}{\|\u\|_2}-1
\geq -0.625r\kappa\tan^2({\theta}),
\end{equation}
which completes the proof.
\end{proof}

We now use the previous lemma to bound the semi-random robustness of the classifier, i.e. $\| \r_{\S}^k \|_2$, to the worst-case robustness $\| \r^k \|_2$ in the case where the curvature is sufficiently small.

\begin{reptheorem}{thm:mainThm_constantCurvature_binary}
Let $\S$ be a random $m$-dimensional subspace of $\mathbb{R}^d$. Define $\alpha:=\sqrt{\nicefrac{m}{d}}$, and let $\kappa := \kappa(\bo_{k})$. Assuming that
$\kappa \leq \frac{C \alpha^2}{\gatwo \|\r^k\|_2}$, the following inequalities hold between $\| \r_\S^k \|_2$ and the worst-case perturbation $\| \r^k \|_2$
\begin{equation}
    \frac{\gaone}{\alpha^2} \left(1- \frac{C_1 \| \r^k \|_2 \kappa \gatwo}{\alpha^2} \right)^2
\leq\frac{\|\r_\S^k \|_2^2}{\|\r^k\|_2^2}
\leq \frac{\gatwo}{\alpha^2}\left(1+\frac{C_2 \|\r^k\|_2\kappa \gatwo}{\alpha^2}\right)^2
\label{eq:mainThm_eq}
\end{equation}
with probability larger than $1- 4 \delta$. The constants can be taken $C = 0.2, C_1 = 0.625, C_2 = 2.25$. % \red{I think it should be 1 - 2$\delta$.}
\end{reptheorem}

%\begin{itemize}[leftmargin=*]
%\item $\sqrt{\frac{2\ln n}{m}}\leq\epsilon\leq\frac{1}{2}$,
%\item $\kappa\leq\frac{\alpha^2}{15\|\r^*\|_2}$,
%\end{itemize}
%the following inequalities hold between the norms of semi-random perturbation  $\r_S^*$ and the worst-case perturbation $\r^*$, w.p.g. $1-\frac{2}{n}$.
%\begin{equation}
%\min\left\{\frac{1-2\epsilon}{\alpha}\left(1-\frac{C_1}{2}\right),\frac{3}{8}\sqrt{\frac{2}{\kappa\|\r^*\|_2}}\left(1-3\|\r^*\|_2\kappa\right)\right\}
%\leq\frac{\|\r_\S^* \|_2}{\|\r^*\|_2}
%\leq \frac{1+6\epsilon}{\alpha}\left(1+\frac{C_2\|\r^*\|_2\kappa}{\alpha^2}\right)
%%+O(\epsilon)\|\r^*\|_2\kappa
%\end{equation}
%\label{prop:bounds}
%\end{proposition}
\begin{figure}
\center
\subfigure[]{
\includegraphics[scale=1]{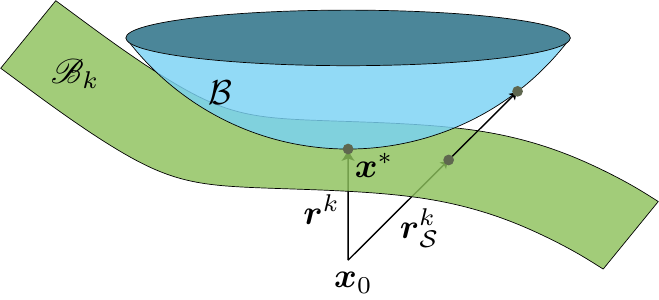}
}
\subfigure[]{
\includegraphics[scale=1]{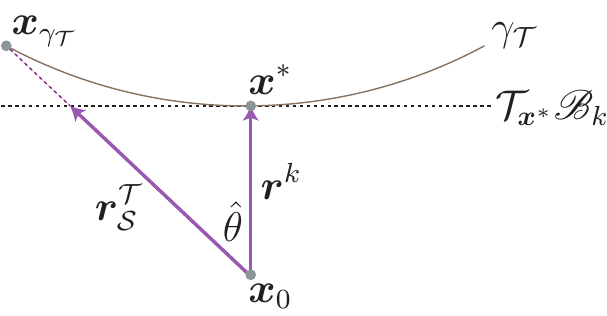}
}
% \caption{$\Gamma_\mathcal{T}$ cross-section.}
\caption{Left: To prove the upper bound, we consider a ball $\mathcal{B}$ included in $\mathcal{R}_k$ that intersects with the boundary at $\x^*$. Upper bounds on $\| \r_{\S}^k \|_2$ derived when the boundary is $\partial \mathcal{B}$ are also valid upper bounds for the real boundary $\bo_k$. Right: Normal section to the decision boundary $\bo_k = \partial \mathcal{B}$ along the normal plane $\mathcal{U} = \text{span} \left( \r_{\S}^{\T}, \r^k\right)$. We denote by $\gamma$ the normal section of boundary $\bo_k$, along the plane $\mathcal{U}$, and by $\T_{\x^*} \bo_k$ the tangent space to the sphere $\partial \mathcal{B}$ at $\x^*$.} % By definition, the curve $\gamma$ lies \textit{below} the curve $\gamma_{\T}$. \red{to change this figure}}
\label{fig:subadv_2d}
\end{figure}
% \red{maybe here a figure showing the normal section taken from above...}

%  we consider the ball $\textfrak{B}$ of maximal radius $R := \textswab{r}_{\infdivx{k}{\lab(\x_0)}} (\x^*)$ 

\begin{proof}[Proof of upper bound]

Denote by $\x^*$ the point belonging to the boundary $\bo_k$ that is closest to the original data point $\x_0$. % To prove the upper bound, let $\mathcal{B}$ 
By definition of the curvature $\kappa$ (see Eq. \ref{eq:def_RBij}), there exists a point $\boldsymbol{z}^*$  such that the ball $\mathcal{B}$ centered at $\z^*$ and of radius $1/\kappa = \| \z^* - \x^* \|_2$ is inscribed in the region $\mathcal{R}_k = \{ x \in \R^d: f_k (\x) > f_{\lab(\x_0)} (\x) \}$ (see Fig. \ref{fig:subadv_2d} (a)).\footnote{For a fixed point $\x^*$ on the boundary, the maximal radius $1/\kappa$ might not be achieved. To prove the result in the general case where the supremum is not achieved, one can consider instead a sequence $(\kappa_n)_n$ converging to $\kappa$, such that the balls of radius $1/\kappa_n$ and intersecting the boundary at $\x^*$ are included in $\mathcal{R}_k$. The same proof and results follow by taking the limit on the bounds derived with ball of radius $1/\kappa_n$.}
% The existence of such a ball follows directly from the definition of the decision boundary curvature $\kappa$ (Eq. (\ref{eq:})). 
% Note that the sphere $\partial \mathcal{B}$ of radius $1/\kappa$ is necessarily tangent to the boundary $\bo_k$ at $\x^*$ (see Fact \ref{fact:semirandom_tangent_ball}). % We assume 
% be an open ball of center $\z^*$, and of radius $\nicefrac{1}{\kappa} = \| \z^* - \x^* \|_2$ that is inscribed in the region $\mathcal{R}_k = \{ x \in \R^d: f_k (\x) > f_{\lab(\x_0)} (\x) \}$  
% We recall that $\kappa$ is the curvature of the decision boundary $\bo_k$, and is equal to the inverse of the radius of the largest open ball one can inscribe in the region $\mathcal{R}_k$.

Observe that the worst-case perturbation along any subspace $\S$ that reaches the ball $\mathcal{B}$ is larger than the perturbation along $\S$ that reaches the region $\mathcal{R}_k$, as $\mathcal{B} \subseteq \mathcal{R}_k$. Therefore, any upper bound derived when the boundary is the sphere of radius $1/\kappa$; i.e., $\bo_k = \partial \mathcal{B}$ is also a valid upper bound for boundary $\bo_k$ (see Fig. \ref{fig:subadv_2d} (a)). It is therefore sufficient to derive an upper bound in the worst case scenario where the boundary $\bo_k = \partial \mathcal{B}$, and we consider this case for the remainder of the proof of the upper bound. % We therefore assume in the remaining of this proof that $\bo_k = \partial \mathcal{B}$.

% We consider the case where the boundary $\bo_k$ coincides with the sphere of radius $R$, $\partial \mathcal{B}$. In fact, 

% \red{What is adherence} Formally, we have % have for any subspace $\S$, 

%\begin{align*}
%\| \r_\S^k \|_2 = \min_{\r \in \S} \| \r \|_2 \text{ s.t. } f_k(\x_0 + \r) \geq f_{\lab(\x_0)} (\x_0 + \r) \leq \min_{\r \in \S} \| \r \|_2 \text{ s.t. } \x_0 + \r \in \mathcal{\overline{B}}.
%\end{align*}

% Since the ball $\mathcal{B}$ is included in the region $\mathcal{R}_k$, it suffices to consider the case where worst case where 
% It should be noted that any upper bound on the quantity of interest $\nicefrac{\| \r_\S^k \|_2}{\| \r^k \|_2}$ 

We now consider the linear classifier whose boundary is tangent to $\bo_k$ at $\x^*$. For the random subspace $\S$, we denote by $\r_{\S}^{\T}$ the worst-case subspace perturbation for this linear classifier. We then focus on the intersection between the boundary $\bo_k$ and the two-dimensional plane $\mathcal{U}$ spanned by the vectors $\r^k$ and $\r_{\S}^\T$. This \textit{normal} section of the boundary cuts the ball $\mathcal{B}$ through its center as the tangent spaces of the decision boundary and the ball coincide. See Fig.~\ref{fig:subadv_2d} for a clarifying figure of this two-dimensional cross-section. We define the angle $\hat{\theta}$ as denoted in Fig.~\ref{fig:subadv_2d}, such that $\cos(\hat{\theta}) = \frac{\| \r^k \|_2}{\| \r^{\mathcal{T}}_{\S} \|_2}$.

We apply our result on linear classifiers in Theorem \ref{th:linear} for the tangent classifier. We have
\begin{align}
\label{eq:b40}
\frac{1}{\cos(\hat{\theta})^2} = \frac{\| \r_\S^\T \|_2^2}{\| \r^k \|_2^2} \leq \frac{1}{\alpha^2} \gatwo,
\end{align}
with probability exceeding $1-2 \delta$. Hence, using $\tan^2(\hat{\theta}) \leq (\cos^2(\hat{\theta}))^{-1}$ and the assumption of the theorem, we deduce that $$\tan^2(\hat{\theta}) \leq \frac{1}{\alpha^2} \gatwo \leq \frac{0.2}{\kappa \| \r^k \|_2},$$ with probability exceeding $1-2 \delta$. Note moreover that $$\| \r^k \|_2 \kappa \leq \frac{0.2 \alpha^2}{\gatwo} < 1.$$
Hence, the assumptions of Lemma \ref{lem:geometric} hold with probability larger than $1-2 \delta$. % , as $\kappa \geq 1/R$ by definition of the global curvature. 
Using the notations of Fig.~\ref{fig:subadv_2d}, we therefore obtain from Lemma \ref{lem:geometric}
% $\tan^2(\hat{\theta}) \leq \frac{0.2}{\| \r^k \|_2 \kappa}$
\begin{equation}
\label{eq:b41}
\frac{\|\x_{\gamma}-\x_0\|_2}{\|\r_\S^\mathcal{T}\|_2}-1\leq C_2 \kappa \|\r^k\|_2\tan^2(\hat{\theta}) % \leq C_2\|\r^k\|_2\kappa\tan^2(\hat{\theta}).
\end{equation}
with probability larger than $1-2 \delta$. 

% Using the fact that $\r^k_\S$ is the minimal perturbation together with the inclusion of the ball $\textswab{B}$ in the region $\mathcal{R}_k$ (see Fig. \ref{fig:subadv_2d}), we have $\|\r^k_\S\|_2\leq\|\x_{\gamma_\mathcal{T}}-\x_0\|_2$. We therefore obtain
Observe that $\| \x_{\gamma} - \x_0 \|_2 \geq \| \r_{\S}^k \|_2$, and that $\tan^2(\hat{\theta}) \leq \frac{\| \r_\S^T \|_2^2}{\| \r^k \|_2^2}$. Hence, we obtain by re-writing Eq. (\ref{eq:b41})
\begin{equation}
\Pbb \left( \frac{\| \r_{\S}^k \|^2_2}{\| \r^k \|^2_2} \leq \left\{ 1 + C_2 \kappa \| \r^k \|_2 \frac{\| \r_\S^\T \|_2^2}{\| \r^k \|_2^2} \right\}^2 \frac{\| \r_\S^\T \|_2^2}{\| \r^k \|_2^2} \right) \geq 1 - 2 \delta.
\end{equation}
Using the inequality in Eq. (\ref{eq:b40}), we obtain
\begin{align*}
\Pbb \left( \frac{\| \r_{\S}^k \|_2^2}{\| \r^k \|_2^2} \leq \left\{ 1 + C_2 \kappa \| \r^k \|_2 \frac{\gatwo}{\alpha^2}  \right\}^2 \frac{\gatwo}{\alpha^2} \right) \geq 1 - 2 \delta,
\end{align*}
which concludes the proof of the upper bound.
\end{proof}

% Note that the above upper bound can be derived by taking any ball $\mathcal{B}$ included in $\mathcal{R}_k$ and tangent to $\x^*$. 
% which concludes the proof of the upper bound.
% Hence, we obtain
%\begin{equation}
%\Pbb \left( \frac{\|\r_\S^k\|_2}{\|\r_\S^\mathcal{T}\|_2}-1\leq
%C_2\frac{\|\r^k\|_2}{R}\tan^2(\hat{\theta}) \right) \geq 1-\delta.
%\end{equation}
%Using once again the result in Theorem \ref{th:linear}, we have $\Pbb \left( \tan^2(\hat{\theta}) \leq \frac{\gatwo}{\alpha^2} \right) \geq 1-\delta$,  we obtain
%\begin{align}
%\Pbb \left( \frac{\|\r_\S^k\|_2}{\|\r_\S^\mathcal{T}\|_2}-1\leq
%C_2\|\r^k\|_2 \frac{\gatwo}{R \alpha^2}  \right) \geq 1-\delta.
%\end{align}
%Finally, using the bound
%\begin{align}
%\Pbb\left( \| \r_{\S}^{\mathcal{T}} \|_2 \leq \| \r^k \|_2 \frac{\sqrt{\gatwo}}{\alpha} \right) \geq 1-\delta,
%\end{align}
%we conclude that
%\begin{align}
%\frac{\| \r_{\S}^k \|_2^2}{\| \r^k \|_2^2} \leq \frac{\gatwo}{\alpha^2} \left( 1 + \frac{C_2 \| \r^k \|_2 \kappa \gatwo}{\alpha^2} \right)^2.
%\end{align}
%with probability exceeding $1-\delta$. % \red{maybe here also no need to do union bound.}

% The worst case scenario for the lower bound is when the curves $\gamma^*$ and 
% \red{$r_{\mathcal{S}}^k$ is for the sphere or the decision boundary?} 

\begin{proof}[Proof of the lower bound]
We now consider the ball $\mathcal{B}'$ of center $\z^*$ and radius $1/\kappa = \| \z^* - \x^* \|_2$ that is included in the region $\mathcal{R}_{\lab (\x_0)}$. Since the ball $\mathcal{B}'$ is, by definition, included in the region $\mathcal{R}_{\lab(\x_0)}$, the worst-case scenario for the lower bound on $\| \r_{\S}^k \|_2$ occurs whenever the decision boundary $\bo_k$ coincides with the ball $\mathcal{B}'$ (see Fig. \ref{fig:subadv_2d_*} (a)). We consider this case in the remainder of the proof. % (denoted $\gamma^*$ in the figure) coincides with the normal section of the decision boundary $\bo_k$.

% This worst-case scenario is depicted in Fig. \ref{fig:subadv_2d_*}, and we consider in the remaining of this proof this worst-case scenario.

To derive the lower bound, we consider the cross-section $\mathcal{U}'$ spanned by the vectors $\r_{\S}^k$ and $\r^k$ (Fig.~\ref{fig:subadv_2d_*} (b)).  
% $\r_{\S}^k$ is the distance from $\x_0$ to the ball 
We have $\| \r^k \|_2 \kappa < 1$; using the lower bound of Lemma \ref{lem:geometric}, we obtain
% Similarly to the upper bound, the assumption of Lemma \ref{lem:geometric} holds with probability larger than $1-\delta$ \red{I don't think this is true and actually needed}
\begin{equation}
-C_1 \kappa \|\r^k\|_2 \tan^2(\tilde{\theta}) \leq \frac{\|\r_\S^k\|_2}{\|\x_\mathcal{T}-\x_0\|_2}-1
\end{equation}
for any $\S$. Observe moreover that
\begin{align*}
\tan^2(\tilde{\theta}) \leq \frac{1}{\cos(\tilde{\theta})^2} = \frac{\| \x_{\T} - \x_0 \|_2^2}{\| \r^k \|_2^2}.
\end{align*}
Hence, the following bound holds:
\begin{align*}
\frac{\| \x_\T - \x_0 \|_2^2}{\| \r^k \|_2^2} \left( 1 - C_1 \kappa \| \r^k \|_2 \frac{\| \x_{\T} - \x_0 \|_2^2}{\| \r^k \|_2^2} \right)^2 \leq \frac{\| \r_\S^k \|_2^2}{\| \r^k \|_2^2}.
\end{align*}
Let $\r_\S^\mathcal{T}$ denote the worst-case perturbation belonging to subspace $\S$ for the \textit{linear} classifier $\mathcal{T}_{x^*} \B_k$. It is not hard to see that $\r_\S^\mathcal{T}$ is \textit{collinear} to $\r_{\S}^k$ (see Lemma \ref{lem:calotte3} for a proof). Hence, we have $\r_{\S}^\T = \x_{\T} - \x_0$. By applying our result on linear classifiers in Theorem \ref{th:linear} for the tangent classifier $\mathcal{T}_{x^*} \B_k$, we have:
\begin{align*}
\Pbb\left( \frac{\gaone}{\alpha^2} \leq \frac{\| \r_{\S}^\T \|_2^2 }{\| \r^k \|_2^2} \leq \frac{\gatwo}{\alpha^2} \right) \geq 1-2 \delta.
\end{align*}
We therefore conclude that
\begin{align*}
\Pbb\left( \frac{\gaone}{\alpha^2} \left\{ 1 - C_1 \kappa \| \r^k \|_2 \frac{\gatwo}{\alpha^2} \right\}^2 \leq \frac{\| \r_{\S}^{k} \|_2^2}{\| \r^k \|_2^2} \right) \geq 1-2 \delta,
\end{align*}
which concludes the proof of the lower bound.
% It is not hard moreover to see that the subspace adversarial perturbation in $\mathcal{S}$ that 

\begin{figure}
\center
\subfigure[]{
\includegraphics[scale=1]{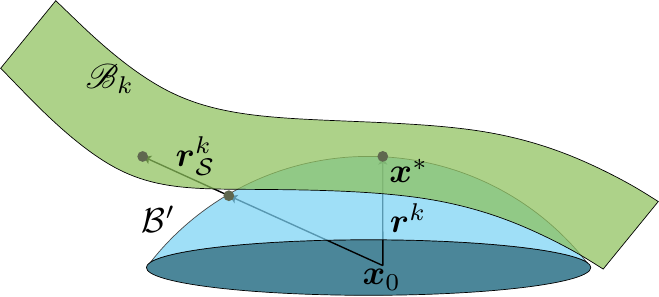}
}
\subfigure[]{
\includegraphics[scale=1]{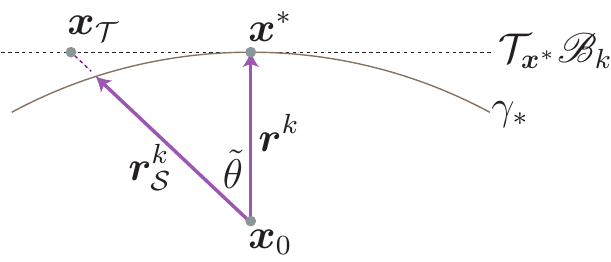}
}
\caption{Left: To prove the lower bound, we consider a ball $\mathcal{B}'$ included in $\mathcal{R}_{\lab(\x_0)}$ that intersects with the boundary at $\x^*$. Lower bounds on $\| \r_{\S}^k \|_2$ derived when the boundary is the sphere $\partial \mathcal{B}'$ are also valid lower bounds for the real boundary $\bo_k$. Right: Cross section of the problem along the plane $\mathcal{U}' = \text{span} \left( \r_{\S}^k, \r^k \right)$. $\gamma$ denotes the normal section of $\bo_k = \mathcal{B}'$ along the plane $\mathcal{U}'$.}
\label{fig:subadv_2d_*}
\end{figure}

\end{proof}

% MULTICLASS
The goal is now to extend the previous result, derived for binary classifiers, to the multiclass classification case. To do so, we show the following lemma.

\begin{lemma}[Binary case to multiclass]
Let $p = \arg\min_{i} \| \r^i \|_2$.
Define the deterministic set
\begin{align}
A = \left\{ k: \| \r^k \|_2 \geq 1.45 \sqrt{\gatwo} \sqrt{\frac{d}{m}} \| \r^* \|_2 \right\}.
\end{align}
Assume that, for all $k \in A^c$, we have
\begin{align}
\Pbb\left( l \leq \frac{\| \r_\S^k \|_2}{\| \r^k \|_2} \leq u \right) \geq 1 - \delta.
\end{align}
and that
\begin{align}
\Pbb\left( \| \r_{\S}^p \|_2 \geq 1.45 \sqrt{\gatwo} \sqrt{\frac{d}{m}} \| \r^* \|_2 \right) \leq t.
\end{align}
Then, we have
\begin{align}
\Pbb\left( l \leq \frac{\| \r_\S^* \|_2}{\| \r^* \|_2} \leq u \right) \geq 1 - (\numClass + 1) \delta - t.
\end{align}
\label{lem:single_class_multiclass_nonLin}
\end{lemma}
\begin{proof}
Note first that
\begin{align}
\Pbb\left( \frac{\| \r_{\S}^* \|_2 }{ \| \r^* \|_2 } \geq u \right) & \leq \Pbb\left( \left\{ \frac{\| \r_{\S}^p \|_2 }{ \| \r^p \|_2 } \geq u \right\} \right) \leq \delta.
\end{align}
We now focus on bounding the other bad event probability $\Pbb \left( \frac{\| \r_{\S}^* \|_2 }{ \| \r^* \|_2 } \leq l \right)$. We have
\begin{align}
\Pbb \left( \frac{\| \r_{\S}^* \|_2 }{ \| \r^* \|_2 } \leq l \right) = \Pbb \left( \min_{k \notin A} \| \r_{\S}^k \|_2 = \| \r_{\S}^* \|_2, \frac{\| \r_{\S}^* \|_2 }{ \| \r^* \|_2 } \leq l \right) + \Pbb \left( \min_{k \in A} \| \r_{\S}^k \|_2 = \| \r_{\S}^* \|_2,  \frac{\| \r_{\S}^* \|_2 }{ \| \r^* \|_2 } \leq l \right)
\end{align}
The first probability can be bounded as follows:
\begin{align}
\Pbb \left( \min_{k \notin A} \| \r_{\S}^k \|_2 = \| \r_{\S}^* \|_2, \frac{\| \r_{\S}^* \|_2 }{ \| \r^* \|_2 } \leq l \right) \leq \Pbb \left( \bigcup_{k \notin A} \frac{\| \r_{\S}^* \|_2}{\| \r^* \|_2} \leq l \right) \leq \numClass \delta.
\end{align}
The second probability can also be bounded in the following way
\begin{align}
\Pbb \left( \min_{k \in A} \| \r_{\S}^k \|_2 = \| \r_{\S}^* \|_2,  \frac{\| \r_{\S}^* \|_2 }{ \| \r^* \|_2 } \leq l \right) \leq \Pbb \left( \min_{k \in A} \| \r_{\S}^k \|_2 = \| \r_{\S}^* \|_2 \right) = \Pbb \left( \exists k \in A, \| \r_{\S}^k \|_2 \leq \| \r_{\S}^* \|_2 \right).
\end{align}
Observe that, for $k \in A$, we have $\| \r_{\S}^k \|_2 \geq \| \r^k \|_2 \geq 1.45 \sqrt{\gatwo} \sqrt{\frac{d}{m}} \| \r^* \|_2$. Hence, we conclude that
\begin{align}
\Pbb \left( \min_{k \in A} \| \r_{\S}^k \|_2 = \| \r_{\S}^* \|_2,  \frac{\| \r_{\S}^* \|_2 }{ \| \r^* \|_2 } \leq l \right) & \leq \Pbb \left( 1.45 \sqrt{\gatwo} \sqrt{\frac{d}{m}} \| \r^* \|_2 \leq \| \r_{\S}^* \|_2 \right)  \\
& \leq \Pbb \left( 1.45 \sqrt{\gatwo} \sqrt{\frac{d}{m}} \| \r^* \|_2 \leq \| \r_{\S}^p \|_2\right)  \leq t.
\end{align}
\end{proof}

\begin{repcorollary}{corr:nonlinear}
Let $\S$ be a random $m$-dimensional subspace of $\mathbb{R}^d$. Assume that, for all $k \notin A$, we have
\begin{align}
\kappa(\bo_k) \|\r^k\|_2 \leq \frac{0.2}{\gatwo} \frac{m}{d}
\end{align}
Then, we have
\begin{equation}
% \min\left\{0.65 \sqrt{\gaone} \sqrt{\frac{d}{m}} , 0.5 \sqrt{\frac{d}{m}} (1-\frac{0.6 m}{d \gatwo}) \right\} \leq\frac{\|\r_\S^* \|_2}{\|\r^*\|_2} \leq 1.45 \sqrt{\gatwo} \sqrt{\frac{d}{m}}
0.875 \sqrt{\gaone} \sqrt{\frac{d}{m}} \leq\frac{\|\r_\S^* \|_2}{\|\r^*\|_2} \leq 1.45 \sqrt{\gatwo} \sqrt{\frac{d}{m}}
\end{equation}
with probability larger than $1-4 (L+2) \delta$.
\end{repcorollary}
\begin{proof}
Using Theorem \ref{thm:mainThm_constantCurvature_binary}, we have that for all $k \notin A$, the result in Eq. (\ref{eq:mainThm_eq}) holds. We simplify the result with the assumption $\kappa(\bo_k) \| \r \|_2 \leq \frac{0.2}{\gatwo} \frac{m}{d}$. Hence, the bounds of Theorem \ref{thm:mainThm_constantCurvature_binary} are given as follows
\begin{align}
\frac{\gaone}{\alpha^2} \left(1 - 0.2 C_1 \right)^2 \leq \frac{\| \r_{\S}^k \|_2^2}{\| \r^k \|_2^2 } \leq \frac{\gatwo}{\alpha^2} \left(1 + 0.2 C_2 \right)^2, % = \gatwo \frac{d}{m} 1.45^2
\end{align}
which leads to the following bounds:
\begin{align}
\gaone \frac{d}{m} 0.875^2 \leq \frac{\| \r_{\S}^k \|_2^2}{\| \r^k \|_2^2 } \leq \gatwo \frac{d}{m} 1.45^2,
\end{align}
with probability exceeding $1-4\delta$.
%For the lower bound, observe that the first term can be lower bounded as follows:
%\begin{align*}
%\frac{\gaone}{\alpha^2} \left( 1 - \frac{C_1}{2} \right)^2 \geq 0.47 \gaone \frac{d}{m}.
%\end{align*}
%The second term can also be bounded using the constraint on $\kappa$:
%\begin{align*}
%C_3 \frac{\left( 1 - 3 \| \r^k \|_2 \kappa )^2\right)}{\kappa \| \r^k \|^2} \geq C_3 \gatwo \frac{\left( 1 - 3 \| \r^k \|_2 \kappa )\right)^2}{C \alpha^2} \geq C_3 \gatwo \frac{\left( 1 - 3 \frac{C \alpha^2}{\gatwo}\right)^2}{C \alpha^2}.
%\end{align*}
%Using the fact that $\alpha^2 / \gatwo \leq 1$, we conclude that
%\begin{align*}
%C_3 \frac{\left( 1 - 3 \| \r^k \|_2 \kappa )^2\right)}{\kappa \| \r^k \|^2} \geq \frac{C_3 \gatwo}{C \alpha^2} \left( 1 - 3 C \right)^2 = 0.9 \gatwo \frac{d}{m}.
%\end{align*}
%For the lower bound, we use once again the bound on the curvature, and show that each term in the lower bound is larger than $0.47 \gaone \frac{d}{m}$. Hence, taking the square root, we obtain for all $k \notin A$,
%% We simplify the result by taking $\epsilon = 1/4$ and $\kappa(\bo_k) \| \r_k \|_2 = \frac{1}{15} \frac{m}{d}$.
%% With these values, we obtain, for all $k \notin A$,
%\begin{align}
%0.65 \sqrt{\frac{d}{m}} \leq \frac{\|\r_\S^k \|_2}{\|\r^k\|_2}
%\leq 1.45 \sqrt{\frac{d}{m}}.
%\label{eq:simplif}
%\end{align}

By using Lemma \ref{lem:single_class_multiclass_nonLin}, together with the fact that $t = \delta$, we obtain
\begin{align}
\Pbb \left( 0.875\sqrt{\gaone}  \sqrt{\frac{d}{m}} \leq \frac{\| \r_{\S}^* \|_2}{\| \r^* \|_2} \leq 1.45 \sqrt{\gatwo} \sqrt{\frac{d}{m}} \right) \geq 1 - 4(L+2) \delta,
\end{align}
which concludes the proof.
\end{proof}

\subsection{Useful results}
\label{subsec:useful}
\begin{figure}[ht]
\centering
\includegraphics[scale=1]{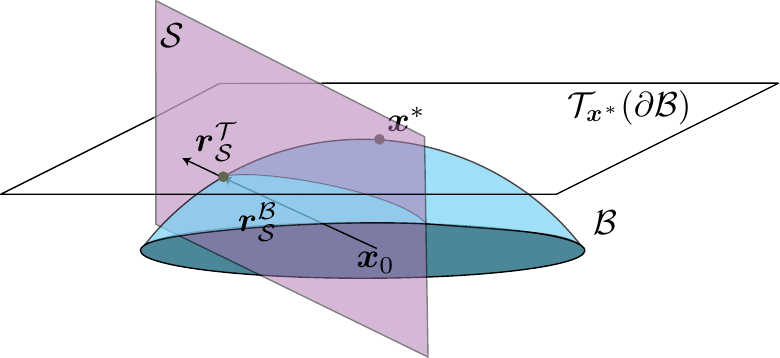}
\caption{\label{fig:calotte3} The worst-case perturbation in the subspace $\mathcal{S}$ when the decision boundary is $\partial \mathcal{B}$ and $T_{\x^*} (\partial \mathcal{B})$ (denoted respectively by $\boldsymbol{r}_{\mathcal{S}}^{\mathcal{B}}$ and $\boldsymbol{r}_{\mathcal{S}}^{\mathcal{T}}$) are collinear.}
\end{figure}
\begin{lemma}
\label{lem:calotte3}
Let $\x_0 \in \R^d$, and $\x^*$ denote the closest point to $\x_0$ on the sphere $\partial \mathcal{B}$ (see Fig. \ref{fig:calotte3}). Let $\T_{\x^*} (\partial \mathcal{B})$ be the tangent space to $\partial \mathcal{B}$ at $\x^*$. For an arbitrary subspace $\S$, let $\r_{\S}^\T$ and $\r_{\S}^\mathcal{B}$ denote the worst-case perturbations of $\x_0$ on the subspace $\S$, when the decision boundaries are respectively $\T_{\x^*} (\partial \mathcal{B})$ and $\partial \mathcal{B}$. Then, the two perturbations $\r_{\S}^\T$ and $\r_{\S}^{\mathcal{B}}$ are collinear.
\end{lemma}
\begin{proof}
Assuming the center of the ball $\mathcal{B}$ is the origin, the points on the sphere $\partial \mathcal{B}$ satisfy equation: $\| \x \|_2 = R$, where $R$ denotes the radius. Hence, the perturbation $\r_{\S}^{\mathcal{B}}$ is given by
\begin{align}
\label{eq:equation_sphere}
\r_{\S}^{\mathcal{B}} = \argmin_{\r \in \mathbb{R}^d} \| \r \|_2^2 \text{ such that } \| \x_0 + \P_\S \r \|_2^2 = R^2.
\end{align}
By equating the gradient of Lagrangian of the above constrained optimization problem to zero, we obtain the following necessary optimality condition
\begin{align*}
\r + \lambda \P_\S (\x_0 + \P_{\S} \r) = 0.
\end{align*}
It should further be noted that $\P_{\S} \r_{\S}^\mathcal{B} = \r_{\S}^\mathcal{B}$. Indeed, if $\r_{\S}^\mathcal{B}$ had a component orthogonal to $\S$, the projection of $\r_{\S}^\mathcal{B}$ onto $\S$ would have strictly lower $\ell_2$ norm, while still satisfying the condition in Eq.(\ref{eq:equation_sphere}). Hence, the necessary condition of optimality becomes
\begin{align*}
(1+\lambda) \r + \lambda \P_{\S} \x_0 = 0,
\end{align*}
from which we conclude that $\r_{\S}^\mathcal{B}$ is collinear to $\P_{\S} \x_0$. 

It should further be noted that $\r_{\S}^\T$ can be computed in closed form, and is collinear to $\P_{\S} (\x^* - \x_0)$, which is itself collinear to $\x_0$, as the the center of the ball was assumed to be the origin. This concludes the proof.
\end{proof}

\begin{lemma}
If $x\in[0,2(\sqrt{2}-1)]$,
\begin{equation}
\sqrt{1-x}\geq 1-\frac{x}{2}-\frac{x^2}{4}.
\end{equation}
\label{lem:ineq}
\end{lemma}

\begin{lemma}
If $x \geq 0$,
\begin{equation}
\sqrt{1+x}\geq 1+\frac{x}{2}-\frac{x^2}{8}.
\end{equation}
\label{lem:ineq2}
\end{lemma}

\end{document}